\documentclass{article}

\usepackage{microtype}
\usepackage{graphicx}
\usepackage{subcaption}
\usepackage{booktabs} %

\usepackage[utf8]{inputenc} %
\usepackage[T1]{fontenc}    %
\usepackage{graphicx} %
\usepackage[colorlinks=True,citecolor=blue,urlcolor=blue,pagebackref=true,backref=true,linkcolor=blue,filecolor=blue]{hyperref}
\renewcommand*{\backrefalt}[4]{%
    \ifcase #1 \footnotesize{(Not cited.)}%
    \or        \footnotesize{(Cited on page~#2.)}%
    \else      \footnotesize{(Cited on pages~#2.)}%
    \fi}

\usepackage{url}            %
\usepackage{booktabs}       %
\usepackage{amsfonts}       %
\usepackage{nicefrac}       %
\usepackage{microtype}      %
\usepackage{xcolor} 
\usepackage{colortbl}
\usepackage{natbib}
\usepackage[table]{xcolor}
\let\oldaddcontentsline\addcontentsline 

\usepackage[accepted]{Styles/icml2026}

\let\addcontentsline\oldaddcontentsline

\usepackage{amsmath}
\usepackage{amssymb}
\usepackage{mathtools}
\usepackage{amsthm}

\usepackage[capitalize,noabbrev]{cleveref}

\usepackage[textsize=tiny]{todonotes}

\newcommand{\mytitle}{%
\textsc{WildCat}: Near-Linear Attention in Theory and Practice
}
\icmltitlerunning{\mytitle}

\usepackage{amssymb,amsthm, amsmath}
\usepackage{thmtools}
\usepackage{url,nicefrac}
\usepackage[ruled,algo2e]{algorithm2e}%
\usepackage{etoolbox}
\usepackage{wrapfig}
\makeatletter
\patchcmd{\@algocf@start}%
  {-1.5em}%
  {0pt}%
  {}{}%
\makeatother
\usepackage{siunitx}
\usepackage{subcaption}   %
\usepackage{float}        %
\usepackage{nameref}
\usepackage{enumitem}
\usepackage[usestackEOL]{stackengine}
\usepackage{booktabs} %
\usepackage{mathtools}
\usepackage{etoc} %
\usepackage{dsfont}
\usepackage{enumitem}  
\usepackage{wrapfig}
\usepackage{graphicx}
\usepackage{multirow}

\usepackage{autonum}

\newtheorem{theorem}{Theorem}
\newtheorem{proposition}{Proposition}
\newtheorem{lemma}{Lemma}
\newtheorem{corollary}{Corollary}
\newtheorem{definition}{Definition}

\newtheorem{remark}{Remark}

\usepackage[capitalize]{cleveref}
\crefname{equation}{}{}
\crefname{lemma}{Lem.}{Lems.}
\crefname{theorem}{Thm.}{Thms.}
\crefname{corollary}{Cor.}{Cors.}
\crefname{algorithm}{Alg.}{Algs.}
\crefalias{algocf}{algorithm}

\crefname{appendix}{App.}{Apps.}
\crefname{example}{Ex.}{Exs.}
\crefname{section}{Sec.}{Secs.}
\crefname{table}{Tab.}{Tabs.}
\crefname{remark}{Rem.}{Rems.}
\crefname{definition}{Def.}{Defs.}
\crefname{proposition}{Prop.}{Props.}
\crefname{figure}{Fig.}{Figs.}

\newcommand\samplempwid{0.16}
\newcommand\samplehinterval{0.05cm}

\newcommand{\eps}{\varepsilon}

\def\balign#1\ealign{\begin{align}#1\end{align}}
\def\baligns#1\ealigns{\begin{align*}#1\end{align*}}
\def\balignat#1\ealign{\begin{alignat}#1\end{alignat}}
\def\balignats#1\ealigns{\begin{alignat*}#1\end{alignat*}}
\def\bitemize#1\eitemize{\begin{itemize}#1\end{itemize}}
\def\benumerate#1\eenumerate{\begin{enumerate}#1\end{enumerate}}

\newenvironment{talign*}
 {\let\displaystyle\textstyle\csname align*\endcsname}
 {\endalign}
\newenvironment{talign}
 {\let\displaystyle\textstyle\csname align\endcsname}
 {\endalign}

\def\balignst#1\ealignst{\begin{talign*}#1\end{talign*}}
\def\balignt#1\ealignt{\begin{talign}#1\end{talign}}
\newcommand{\qtext}[1]{\quad\text{#1}\quad} 
 
\newcommand{\stext}[1]{\ \text{#1}\ } 
\newcommand{\sstext}[1]{\ \ \text{#1}\ \ }

\let\originalleft\left
\let\originalright\right
\renewcommand{\left}{\mathopen{}\mathclose\bgroup\originalleft}
\renewcommand{\right}{\aftergroup\egroup\originalright}

\def\Holder{H\"older\xspace}

\def\Nystrom{Nystr\"om\xspace}

\def\tinycitep*#1{{\tiny\citep*{#1}}}
\def\tinycitealt*#1{{\tiny\citealt*{#1}}}
\def\tinycite*#1{{\tiny\cite*{#1}}}
\def\smallcitep*#1{{\scriptsize\citep*{#1}}}
\def\smallcitealt*#1{{\scriptsize\citealt*{#1}}}
\def\smallcite*#1{{\scriptsize\cite*{#1}}}

\def\mbf#1{\mathbf{#1}}
\def\mbb#1{\mathbb{#1}}
\def\mc#1{\mathcal{#1}}

\def\tbf#1{\textbf{#1}}

\def\reals{\mathbb{R}} %
\def\naturals{\mathbb{N}} %

\def\<{\left\langle} %
\def\>{\right\rangle}

\def\iff{\Leftrightarrow}

\def\defeq{\triangleq} %
\newcommand{\boldone}{\mbf{1}} %
\newcommand{\boldzero}{\mbf{0}} 
\def\norm#1{\|{#1}\|} %
\newcommand{\twonorm}[1]{\norm{#1}_2} %
\newcommand{\opnorm}[1]{\norm{#1}_{\mathrm{op}}} %
\newcommand{\sopnorm}[1]{\staticnorm{#1}_{\mathrm{op}}} %
\newcommand{\fronorm}[1]{\norm{#1}_{F}} %
\def\staticnorm#1{\|{#1}\|} %
\newcommand{\inner}[2]{\langle{#1},{#2}\rangle} %
\def\E{\mbb{E}} %
\def\Earg#1{\E\left[{#1}\right]}

\newcommand{\dd}[1]{\frac{\mathrm d}{\mathrm d #1}}
\newcommand{\pdd}[1]{\frac{\mathrm \partial}{\mathrm \partial #1}}

\providecommand{\diag}{\mathop\mathrm{diag}}

\def\rank#1{\mathrm{rank}({#1})}

\ifdefined\nonewproofenvironments\else
\ifdefined\ispres\else
\fi

\newcommand{\ncref}[1]{\cref{#1}: \nameref*{#1}} %

\newcommand{\pcref}[1]{Proof of \ncref{#1}} %

\newcommand{\tablelineskip}{2.5mm}
\newcommand{\tabletopskip}{-4mm}

\def \wtdattn{\hyperref[alg:weighted-coreset-attention]{\textsc{WtdAttn}}\xspace}

\newcommand{\rpnys}{\hyperref[alg:rpnys]{\textsc{RPNys}}\xspace}
\newcommand{\rpnysnolink}{\textsc{RPNys}\xspace}

\newcommand{\cattnolink}{\textsc{WildCat}\xspace}
\newcommand{\catt}{\hyperref[alg:cmpd-attn]{\textsc{WildCat}}\xspace}
\newcommand{\compresskv}{\hyperref[alg:compresskv]{\textsc{CompressKV}}\xspace}
\newcommand{\compresskvnolink}{\textsc{CompressKV}\xspace}

\def \proj{\mathbf{P}}

\def\paren#1{\left({#1}\right)}

\def\nuc#1{\mathrm{tr}\left({#1}\right)}
\def\norm#1{\left\|{#1}\right\|} %

\newcommand{\maxnorm}[1]{\norm{#1}_{\max}} %
\newcommand{\smaxnorm}[1]{\staticnorm{#1}_{\max}}
\newcommand{\specnorm}[1]{\opnorm{#1}} %
\newcommand{\sspecnorm}[1]{\sopnorm{#1}} %
\newcommand{\rownorm}[1]{\staticnorm{#1}_{2,\infty}} %

\def \hatt{h}
\def \hnys{h_{\mathrm{nys}}}
\def \hres{h_{\mathrm{res}}}

\def \Hinv{\mathbf M}
\def \Hcore{\mathbf R}

\def \queries{\mathbf{q}}
\def \keys{\mathbf{k}}
\def \values{\mathbf{v}}
\def \vmax{\values_{\textrm{max}}}
\def \vmin{\values_{\textrm{min}}}

\def \nout{n_{\mathrm{out}}}

\DeclarePairedDelimiterXPP\LambW[1]{W_0}{(}{)}{}{#1}
\DeclarePairedDelimiterXPP\LambWm[1]{W_{-1}}{(}{)}{}{#1}

\def \CAtt{\widehat{\mathbf O}}

\DeclarePairedDelimiterX{\set}[1]{\{}{\}}{\setargs{#1}}
\NewDocumentCommand{\setargs}{>{\SplitArgument{1}{;}}m}
{\setargsaux#1}
\NewDocumentCommand{\setargsaux}{mm}
{\IfNoValueTF{#2}{#1} {#1\nonscript\:\delimsize\vert\allowbreak\nonscript\:\mathopen{}#2}}%

\def\bigO{O}

\newcommand{\A}{\mbf A}

\newcommand{\B}{\mbf B}
\newcommand{\K}{\mbf K}
\newcommand{\Q}{\mbf Q}
\newcommand{\V}{\mbf V}
\newcommand{\U}{\mbf U}

\newcommand{\Hc}{\mbf H}
\newcommand{\Hhat}[1][r]{\widehat{\mbf H}^{#1}}

\newcommand{\W}{\mbf W}

\newcommand{\T}{\mbf T}

\newcommand{\D}{\mbf{D}}

\newcommand{\hatD}{\mbf{\widehat{D}}}
\newcommand{\hatA}{\widehat{\A}}

\newcommand{\Dinv}{\D^{-1}}

\newcommand{\hatDinv}{\hatD^{-1}}

\newcommand{\coreset}{\K_{\mathcal S}}
\newcommand{\dataset}{\K}

\newcommand{\error}[1][]{\smaxnorm{\mathbf O - \widehat {\mathbf O}_{#1}}}

\def \x{\mathbf{x}}
\def \y{\mathbf{y}}

\def \e{\mathbf{e}}
\def \a{\mathbf{a}}
\def \b{\mathbf{b}}
\def \g{\mathbf{g}}

\def \p{\mathbf{p}}

\def \w{\mathbf{w}}
\def \u{\mathbf{u}}

\def \diag{\mathrm{diag}}

\def \res{\mathrm{res}}

\def \Hres{\mbf{H}_\res}

\def \Knorm{\rownorm{\K}}

\def \Vnorm{\maxnorm{\V}}

\def \RK{R_\K}
\def \RQ{R_\Q}

\DeclarePairedDelimiter{\floor}{\lfloor}{\rfloor}

\newcommand{\dblue}[1]{\textcolor{blue!50!black}{{\textsc{#1}}}}
\newcommand{\bacronym}{\dblue{W}eighted \dblue{I}terative \dblue{L}ow-rank \dblue{D}ecomposition for \dblue{C}oreset \dblue{At}tention\xspace}

\renewcommand{\O}{\mathbf O}
\newcommand{\Ohat}{\widehat {\mathbf O}}
\newcommand{\hatO}{\Ohat}
\newcommand{\Amin}{\A_{\mathrm{min}}}
\newcommand{\clip}{\mathrm{clip}}

\newcommand{\kbar}{\bar{\keys}}
\renewcommand{\o}{\mathbf o}

\begin{document}

\twocolumn[
  \icmltitle{\mytitle}

  \icmlsetsymbol{equal}{*}

  \begin{icmlauthorlist}
    \icmlauthor{Tobias Schr\"oder}{icl}
    \icmlauthor{Lester Mackey}{comp}
  \end{icmlauthorlist}
  \icmlaffiliation{icl}{Imperial College London}
  \icmlaffiliation{comp}{Microsoft Research New
England}
  \icmlcorrespondingauthor{Tobias Schr\"oder}{t.schroeder21@imperial.ac.uk}
  \icmlcorrespondingauthor{Lester Mackey}{lmackey@microsoft.com}
  \icmlkeywords{attention mechanism, KV cache compression, near-linear time, kernel methods, distribution compression, Nyström, weighted coreset}
  \vskip 0.3in
]

\printAffiliationsAndNotice{}  %

\etoctocstyle{1}{Table of contents}
\etocdepthtag.toc{mtchapter}
\etocsettagdepth{mtchapter}{section}
\etocsettocdepth{section} %

\begin{abstract}
We introduce \textsc{WildCat}, a high-accuracy, low-cost approach to compressing the attention mechanism in neural networks. While attention is a staple of modern network architectures, it is also notoriously expensive to deploy due to resource requirements that scale quadratically with the input sequence length $n$. \textsc{WildCat} avoids these quadratic costs by only attending over a small weighted coreset. Crucially, we select the coreset using a fast but spectrally-accurate subsampling algorithm -- randomly pivoted Cholesky -- and weight the elements optimally to minimise reconstruction error. Remarkably, given bounded inputs, \textsc{WildCat} approximates exact attention with super-polynomial $O(n^{-\sqrt{\log(\log(n))}})$ error decay while running in near-linear $O(n^{1+o(1)})$ time. In contrast, prior practical approximations either lack error guarantees or require quadratic runtime to guarantee such high fidelity. We couple this advance with a GPU-optimised PyTorch implementation and a suite of benchmark experiments demonstrating the benefits of \textsc{WildCat} for image generation, image classification, and language model KV cache compression.
\end{abstract}

\section{Introduction}
A central component of transformer-based models \citep{vaswani2017attention} is the attention mechanism, which enables the modelling of long-range dependencies in sequences. The importance of the attention mechanism in today's machine learning landscape cannot be overstated. Practically all large-scale models use this operation, whether in natural language processing (e.g., BERT \citep{devlin2019bert} and GPT \citep{radfordimproving}), image synthesis \citep{esser2021taming}, or protein structure prediction \citep{jumper2021highly}. However, attention is also notoriously expensive to deploy as its resource requirements grow quadratically in the input sequence length $n$. 

This quadratic cost has motivated the development of fast approximate attention methods, both in theory and in practice. In practice, Reformer \citep{kitaevreformer}, for example, reduces runtime by evaluating a sparse subset of the attention weights while Performer \citep{choromanskirethinking} approximates an assumed low-rank structure of the attention matrix and Scatterbrain \citep{chen2021scatterbrain} combines the two approaches. 
Meanwhile, in theory, \citet{alman2023fast} showed that, for suitably-bounded inputs, one can approximate the attention output with fast, polynomial ($O(n^{-t})$ for any $t>0$) error decay in near-linear $O(n^{1+o(1)})$ time. 

However, a substantial gap remains between the theory and practice. To date, only a few works have developed practical attention approximations with correctness guarantees 
\citep{zandieh2023kdeformer, han2024hyperattention, carrell2025low, han2025streaming}, and the best of these (a) require quadratic time for fast, polynomial error decay and (b) only ensure slow, near-constant $n^{-o(1)}$ error decay in near-linear time. 

To bridge the theory-practice gap, we 
introduce \catt (\bacronym),  
a weighted coreset 
approach to approximate attention that is 
simultaneously 
(1) computationally efficient, with a runtime that grows near-linearly in $n$; %
(2) spectrally-accurate, allowing for super-polynomial error decay; %
and
(3) practical, with an efficient GPU-optimised implementation. 
Specifically, as our core contributions, we establish
the following desirable properties for \catt: 
\begin{enumerate}[leftmargin=15pt]
    \item \catt avoids the quadratic cost of exact attention by attending only over a small weighted coreset of $r$ input keys. The keys are selected in $O(nr^2)$ time using a parallelised randomly pivoted Cholesky algorithm \citep{chen2022randomly}, and reweighted optimally to minimise attention reconstruction error in $O(nrd)$ time. Hence, \catt runs in near-linear time whenever $r\in n^{o(1)}$.
    \item Thanks to its selection rule and optimal reweighting, \catt approximates the attention output %
    with near-optimal low-rank-approximation error. 
    As a result, when attention inputs are bounded, a near-constant $r\in n^{o(1)}$ coreset size suffices for \emph{super}-polynomial $O(n^{-\sqrt{\log(\log(n))}})$ error decay. 
    \item More generally, pushing beyond the limits of prior work on the computational hardness of attention \citep{alman2023fast,keles2023computational}, we show that \catt can deliver super-polynomial error decay in near-linear time even when the input entries or dimensions grow super-logarithmically in $n$. 
    \item Our benchmark experiments with image generation and image classification show that \catt can generate higher-quality outputs more quickly than five leading attention approximations.
    \item Our benchmark experiments with $13$ long-context language understanding tasks shows that \catt can also reduce the memory requirements of long-context language models more effectively than five leading KV cache compression methods.
\end{enumerate}

\paragraph{Notation.}
For each $n\in \mathbb N$ we define $[n] \defeq \set{1, 2, \dots, n}$. For a set $\mathcal S$ we write $\lvert \mathcal S\rvert$ for the number of elements in the set. We often treat a matrix $\A \in \mathbb R^{n\times d}$ as a tuple of row vectors $\A = (\a_i)_{i\in[n]}$ with $\a_i \in \mathbb R^d$, and denote sub-selections as $\A_{\mathcal S} = (\a_i)_{i\in \mathcal S}$. For two such ordered sets $\A = (\a_i)_{i\in[n]}$, $\B = (\b_l)_{l\in [m]} \subseteq \mathbb R^d$ and a real-valued function $h:\mathbb R^d\times \mathbb R^d\to \mathbb R$ we write $h(\A, \B) = (h(\a_i, \b_l))_{i\in[n], l\in [m]}$. $\inner{\cdot}{\cdot}:\mathbb R^d\times \mathbb R^d\to \mathbb R$ denotes the Euclidean inner product.
For a symmetric matrix $\Hc \in \mathbb R^{n\times n}$ we denote the pseudo-inverse by $\Hc^+$ and the $r$-th largest eigenvalue by $\lambda_r(\Hc)$. Further, we define for $\A\in \mathbb R^{n\times d}$ the matrix norms $\specnorm{\A} \defeq \sqrt{\lambda_1(\A^\top\A)}$, $\maxnorm{\A} \defeq \max_{i \in [n], j \in [d]} \lvert \A_{ij}\rvert$, and $\rownorm{\A} \defeq \max_{i \in [n]} \lVert \A_{i, :}\rVert_2$.  %
\section{Weighted Coreset Attention}
The softmax attention mechanism takes as input a sequence of queries $\Q \defeq (\mathbf q_i)_{i\in [m]} \in \mathbb R^{m \times d}$, keys $\K \defeq (\mathbf k_l)_{l\in [n]} \in \mathbb R^{n \times d}$, and values $\V \defeq (\mathbf v_l)_{l\in [n]} \in \mathbb R^{n\times d}$ and outputs the \emph{softmax matrix}
\begin{align}\label{eq:attn-output}
    \mathbf O \defeq  \left(\frac{\sum_{l=1}^{n}\exp(\beta\inner{\queries_i}{\keys_l}) \mathbf v_l}{\sum_{l=1}^{n}\exp(\beta\inner{\queries_i}{\keys_l})}\right)_{i\in[m]} 
    =
    \Dinv \A \V %
\end{align}
with attention matrix $\A_{il}\defeq \exp(\beta{\inner{\queries_i}{\keys_l}})$, scaling matrix $\D \defeq  \diag(\A\boldone_n)$, and scale factor $\beta$, often chosen as $\beta = 1/\sqrt{d}$.
The chief bottleneck in attention is the $m\times n$ attention matrix $\A$. When $m$ and $n$ are comparably large---a common occurrence in vision and language modelling---exact computation of $\O$ requires quadratic $\Theta(n^2d)$ runtime simply to evaluate and multiply by $\A$ (using standard matrix multiplication). 

\subsection{Low-rank attention approximation}
Our high-level strategy to reduce this cost is to approximate the softmax matrix $\O$ using a low-rank approximation of $\A$. Notably, for any $\hatA = \U \W$ with $\U\in \mathbb R^{m\times r}$ and $\W \in \mathbb R^{r\times n}$, the plug-in approximation $\hatDinv \hatA \V \defeq \diag(\U \W \boldone_n)^{-1} \U \W \V$
can be computed with $\bigO(mrd+nrd)$ operations and $\bigO((m + n)(r+d))$ memory by multiplying the {weights} $\W$ with $(\V, \boldone_n)$ before applying $\U$. 
This observation, combined with our next result shows that any low-rank rowwise-accurate estimate of the attention matrix $\A$ can be efficiently transformed into an entrywise-accurate estimate of the softmax matrix $\mathbf O$. 
\begin{lemma}[Approximate attention guarantee]\label{prop:transformer-bound-via-A-approximation}
    Let $\hatA$ be an approximation to $\A$,  $\hatD = \mathrm{diag}(\hatA \boldone_n)$, and $\CAtt \defeq \mathrm{clip}(\hatDinv\hatA \V, \values_\mathrm{min}, \values_\mathrm{max})$ for ${\values_{\mathrm{min}}}_j = \min_{l\in [n]} \values_{lj}$ and  ${\values_{\mathrm{max}}}_j = \max_{l\in [n]} \values_{lj}$. Then
\begin{align}\label{eq:clipped-catt-bound-proposition}
\!\!\!\!\!\!\!\!\!\error
    \leq
\Vnorm\min\paren{\frac{\frac{3}{\sqrt{n}}\rownorm{\A-\hatA}}{\displaystyle\min_{i\in[m],j\in[n]}\A_{ij}}, 2}\!.\!\!\!\!\!\!\!
\end{align}
\end{lemma}

In this statement, proved in \cref{subsection:proof:cat-error-decomposition}, we additionally constrain each estimate $\Ohat_{ij}$ to lie in the value range $[\values_{\min j}, \values_{\max j}]$ as the target entry $\O_{ij}$ also satisfies this property. 
Our next step is to identify a high-quality low-rank approximation $\hatA$ to seed our  output estimate $\Ohat$.

\subsection{Optimal weighting via \Nystrom approximation}
A widely used tool for constructing low-rank approximations for symmetric positive definite (s.p.d.) matrices is the Nystr\"om method \citep{williams2000using}. To understand how we can exploit approximations of s.p.d.\ matrices, note that the attention matrix $\A = \exp(\beta\Q\K^\top)$ can be written in terms of the  \emph{exponential kernel function}\footnote{See \cref{sec:kernel-preliminaries} for relevant background on reproducing kernels.} $\hatt(\queries, \keys) \defeq \exp(\beta \inner{\queries}{\keys})$. 
 
 The kernel perspective suggests the following construction of a low-rank approximation of $\A = h(\Q,\K)$. The kernel features $\set{h(\cdot, \keys_l) ; l\in [n]}$ span an (at most) $n$-dimensional vector space $\mathcal H$ with inner product $\langle h(\cdot, \x), h(\cdot, \y)\rangle_{\mathcal H} = h(\x, \y)$. Accordingly, $\set{h(\cdot, \keys_l) ; l\in \mathcal S}$ for a subset $\mathcal S\subseteq [n]$ with $\lvert \mathcal S\rvert = r$ defines an at most $r$-dimensional subspace $\mathcal H_{\mathcal S} \subset \mathcal H$. The orthogonal projection of the kernel features $h(\cdot, \keys_l)$ onto $\mathcal H_{\mathcal S}$ is called a Nystr\"om approximation and takes the form 
\begin{equation}
    \hnys(\cdot, \keys_l) \defeq  \hatt(\cdot, \coreset) \hatt(\coreset, \coreset)^{+} \hatt(\coreset, \keys_l) \in \mathcal H_{\mathcal S}\,.
\end{equation}
The \textit{Nystr\"om weights} $\w_l \defeq h(\coreset, \coreset)^{+}h(\coreset, \keys_l)$ constitute the \textit{optimal} weighting of $h(\cdot, \coreset)$ to minimise the difference between $h(\cdot, \keys_l)$ and $h(\cdot, \coreset)\w$. 
If we adopt the low-rank approximation $\hatA = \hnys(\Q,\K)$, then the approximation error is governed by the residual kernel function $\hres = h-\hnys$ as $\A - \hatA = \hres(\Q,\K)$.
Our next result, proven in \cref{section:proof:nystrom},  provides a precise guarantee for the rowwise error of this \Nystrom-based approximation. %
\begin{lemma}[\Nystrom guarantee]\label{prop:nystrom}
    Let $\mathcal S \subseteq [n]$ be a subset with $\lvert \mathcal S\rvert = r$ and $\K_{\mathcal S} = (\keys_l)_{l\in \mathcal S}$ the associated rows of $\K$. Then, $\hatA \defeq \hatt(\Q, \coreset) \hatt(\coreset, \coreset)^{+} \hatt(\coreset, \K)$ has rank $\leq r$ and satisfies the following guarantee for $\RQ\defeq\rownorm{\Q}$:
    \begin{align}
    \rownorm{\A - \hatA}^2  
    \leq 
    \exp(\beta\RQ^2)\,  %
    \specnorm{\hres(\K, \K)}\,.
    \end{align}
\end{lemma}

\cref{prop:nystrom} shows that, 
to obtain an entrywise-accurate \Nystrom estimate of $\A$, it suffices to accurately approximate the s.p.d.\ key kernel matrix $\Hc \defeq h(\K, \K)$. Since the quality of a \Nystrom approximation is determined by the quality of its coreset $\coreset$, we now turn our attention to coreset selection.

\subsection{Coreset construction with random pivoting}
To select a coreset $\coreset$ algorithmically from $\dataset$ we adapt the randomly pivoted Cholesky (RPC) algorithm of \citet{chen2022randomly}.
RPC builds a partial Cholesky decomposition of the kernel matrix $\Hc$. Central to its guarantees is the pivoting rule, which samples each coreset point from the diagonal of the current residual kernel. We adopt the same pivoting rule but construct the Nystr\"om weights $\W \defeq h(\coreset, \coreset)^{+}h(\coreset, \dataset)$ instead.

Starting with an empty coreset $\mathcal S \gets \emptyset$ and the diagonal of the residual kernel $\hres^{0}(\keys_l, \keys_l) = h(\keys_l, \keys_l)$ for each key $\keys_l$, we sample in each round a pivot index $s \sim \p^r$, where
\begin{equation}\label{eq:rpc-pivoting-rule}
    \p^r_l \defeq \frac{\hres^{r}(\keys_l, \keys_l)}{\sum_{l\in [n]} \hres^{r}(\keys_l, \keys_l)}\ \quad \text{for}\quad l\in [n]\,.
\end{equation}
As long as $\Hc$ is not fully approximated, $\hres^{r}(\keys_s, \keys_s)>0$ by construction, and the kernel matrix associated with the new coreset $\mathcal S' \gets \mathcal S\cup\set{s}$ remains invertible. We maintain $h(\K_{\mathcal S'}, \K_{\mathcal S'})^{-1}$ via rank-one updates: Upon adding a pivot $s$, we use for $\g^\top \defeq \frac{\paren{h(\keys_{s}, \coreset) h(\coreset, \coreset)^{-1}, -1}}{\sqrt{\hres^r(\keys_{s},\keys_{s})}}$ the following recursive relations for the inverse of the kernel matrix
\begin{align} 
    h(\K_{\mathcal S'}, \K_{\mathcal S'})^{-1} = 
    \begin{pmatrix}
        h(\coreset, \coreset)^{-1} & 0 \\
        0 & 0
    \end{pmatrix} + \g \g^\top\,
\end{align}
and the diagonal of the residual kernel for $l\in [n]$
\begin{align}
    \hres^{r+1}(\keys_l, \keys_l) = \hres^{r}(\keys_l, \keys_l) - (\g^\top h(\coreset, \keys_l))^2\,.
\end{align}
We provide a justification for the recursive update rule in \cref{app:proof:prop:kernel-core-inverse-update}. 
The complete algorithm, \rpnys, is summarised in \cref{alg:rpnys}. Note that \rpnys only accesses $\bigO(nr)$ entries of $h(\K, \K)$ and runs in $O(nr^2 + nrd)$ time.

\begin{algorithm2e}[h]
    \caption{Randomly pivoted Nystr\"om (\rpnysnolink)}%
    \label{alg:rpnys}
    \small
    \begin{algorithmic}
        \STATE {\bf Input:} dataset $\K =(\keys_l)_{l\in[n]}$,  kernel $h$,  rank $r$
        \STATE $\p \gets (h(\keys_l, \keys_l))_{l\in [n]}$ \COMMENT{Compute kernel diagonal}
        \STATE $\Hinv \gets \boldzero_{r\times r} \,;\quad \Hcore\gets \boldzero_{r\times n};\quad \g \gets \boldzero_r$
        \STATE $\mathcal S \gets \emptyset$ \COMMENT{Initialize empty coreset}
        \FOR{ $i = 1, \dots, r$ }
            \STATE $\mathcal S \gets \mathcal S \cup \{s\} \qtext{for} s \sim \frac{\p}{\sum_{l=1}^{n} \p_l}$
            \COMMENT{Sample pivot}
            \STATE // Update kernel inverse
            \IF{$i>0$}
            \STATE $\g_{[i-1]} \gets \Hinv_{[i-1], [i-1]}\Hcore_{[i-1], s}$ \COMMENT{$O(i^2)$ operations}
            \ENDIF
            \STATE $\g_{i} \gets -1$;\quad $\g_{[i]} \gets \g_{[i]}/\sqrt{\p_s}$
            \STATE $\Hinv \gets \Hinv + \g \g^\top$
            \STATE // Update pivot distribution
            \STATE $\Hcore_{i, [n]}\gets h(\keys_s, \K\,)$ \COMMENT{$O(nd)$ operations}
            \STATE $\delta \gets \g_{[i]}^\top \Hcore_{[i], [n]}$ \COMMENT{$O(ni)$ operations}
            \STATE $\p \gets \p - \delta^2$ \COMMENT{Entrywise power function}
            \STATE $\p_s \gets 0$ \COMMENT{For numerical stability}
        \ENDFOR
        \STATE {\bf Return:} Coreset $\coreset$; \Nystrom weights $\W\! =\! \Hinv\Hcore$ 
    \end{algorithmic}
\end{algorithm2e}

\begin{figure*}
    \centering
    \includegraphics[width=\textwidth]{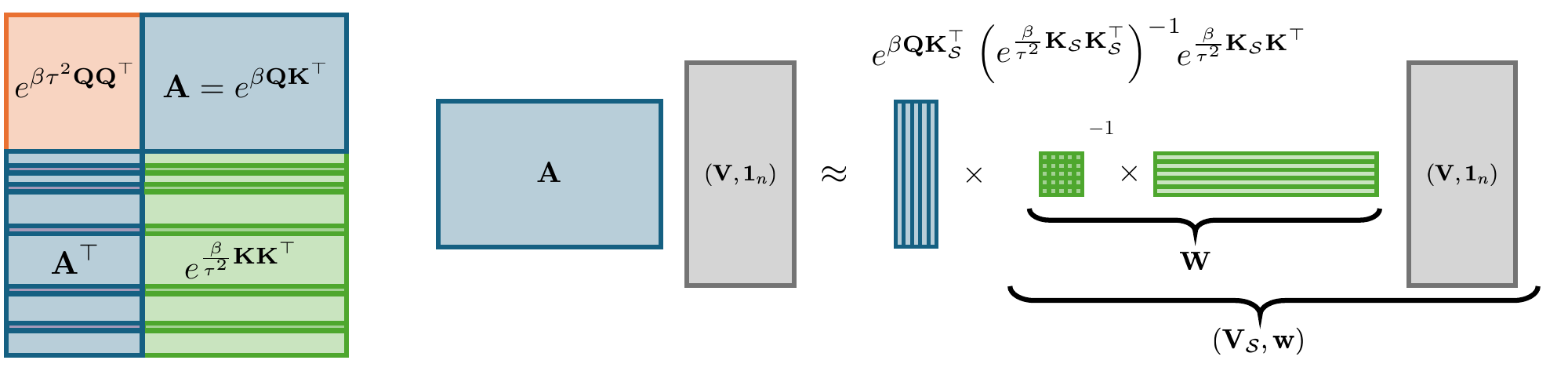}
    \caption[Visualisation of compressing KDEs via a Nystr\"om approximation.]
    {\textbf{Visualisation of the \catt methodology.} Our goal is the approximation of the off-diagonal block $\A$ through a Nystr\"om approximation $\hatA_\tau \defeq h(\Q, \coreset)h\paren{\frac{1}{\tau}\coreset, \frac{1}{\tau}\coreset}^{-1}h\paren{\frac{1}{\tau}\coreset, \frac{1}{\tau}\K}$. With the right order of operations, the computation cost for $\A\V$ decreases from $\bigO (mnd)$ to $\bigO (rnd + mrd + nr^2)$. For the exponential kernel, the off-diagonal block is invariant under $\Q \to \tau\Q$, $\K \to \frac{1}{\tau}\K$. Since we only select coreset points from $\K$, we can optimise for low-rank approximability.}
    \label{fig:compresskv-visualisation}
    \vspace{-.5\baselineskip}
\end{figure*}
In \cref{subsection:proof:rp-cholesky}, we mildly adapt the arguments of \citet[Thm.~7]{epperly2023kernel} to provide the following operator norm guarantee for \rpnys.
\begin{theorem}[\rpnysnolink guarantee]\label{thm:rp-cholesky}
Fix any $\varepsilon > 0$ and consider the coreset $\coreset$ and \Nystrom weights $\W$ outputted by \rpnys (\cref{alg:rpnys}) with kernel $h$, dataset $\K$, and rank parameter $r$.  If $\Hc\defeq h(\K,\K)$ and $\widehat \Hc^{r} \defeq h(\K, \coreset)\W$, then $\mathbb E\sspecnorm{\Hc - \widehat \Hc^{r}} \leq \varepsilon$ whenever, for some $\T\preceq \Hc$, 
\begin{talign}
 r \geq \mathrm{rank}(\mathbf T)\log\left(\frac{\sspecnorm{\Hc}}{\varepsilon}\right) + 
\nuc{\Hc - \mathbf T}\left(\frac{1}{\varepsilon}- \frac{1}{\sspecnorm{\Hc}}\right)\,.
 \end{talign}
\end{theorem}

\cref{thm:rp-cholesky} links the approximation error of the randomly pivoted Nystr\"om method to the approximability of $\Hc$ by \emph{any} low-rank operator $\T\preceq \Hc$. In \cref{prop:polynomial-approximation} we show that an order $s$ Taylor approximation of the exponential function yields an operator $\T^s$ with rank $\leq \binom{s+d}{d}$ and 
\begin{talign}\label{eq:taylor-approximation}
    \nuc{\Hc - \T^s} \leq n\exp(\beta \Knorm^2)\paren{\frac{e\beta\Knorm^2}{s+1}}^{s+1}\,
\end{talign}
The fast decay of this approximation in the order parameter $s$ will allow us to derive fast-decaying error rates for the \rpnys approximation in \cref{section:theory}.

\subsection{Invariance of attention under shift and rescaling}\label{sec:rescaling}
Interestingly, the softmax matrix $\O=(\o_i)_{i\in[m]}$ is invariant under a global recentring of the keys. This follows as 
\begin{talign}
\o_{i}
    \!=\!
\frac{h(\queries_i, \K)\V}{h(\queries_i, \K)\boldone_n}\frac{\exp(-\beta\inner{\queries_i}{\kbar})}{\exp(-\beta\inner{\queries_i}{\kbar})}
    \!=\!
\frac{h(\queries_i, \K-\boldone_n\kbar^\top)\V}{h(\queries_i, \K-\boldone_n\kbar^\top)\boldone_n}
\end{talign}
for any row vector $\kbar\in\mathbb R^d$ and each $i\in[m]$. 
Hence our approximation algorithms are also free to operate on  recentred keys. 
Hereafter, we choose $\kbar\defeq\frac{1}{n}\sum_{i=1}^n\keys_i$ and treat $\K$ as the matrix of recentred keys, $\keys_i - \kbar$.  %

Note also that the attention matrix $\A$ is invariant under rescalings of the keys and queries, i.e., we are allowed to rescale $\K \gets \tau^{-1}\K$ and $\Q \gets \tau\Q$ without changing $\A$. However, when we form our approximation of $\A$ using rescaled keys and queries, the approximation error is affected in two ways. On the one hand, increasing $\tau$ makes the data matrix $\Hc_\tau \defeq h(\tau^{-1}\K, \tau^{-1}\K)$ increasingly low-rank approximable. In fact, in the extreme case $\Hc_\tau \xrightarrow{\tau\to\infty} \boldone_n \boldone_n^\top$ becomes a rank one matrix. On the other hand, $\tau$ increases the query-based error inflation factor $\exp(\beta\tau^2 \RQ^2)$ of \cref{prop:nystrom}. 
For $\RK \defeq \rownorm{\K}$, 
our theory in \cref{section:theory} suggests the following closed-form rescaling parameter that reflects the asymmetric roles of the keys and queries in our attention reconstruction:
\begin{talign}\label{eq:kernel-temperature}
    \tau \defeq \sqrt{\frac{\RK}{\RQ}\frac{b_0}{2\LambW*{{b_0/}{(2\rho_0)}}}} \quad\text{with}\quad b_0\defeq \frac{\log(n)}{\beta\RQ\RK} + 2\,.
\end{talign}
Here, $z \mapsto \LambW{z}$ denotes the Lambert-W function which is defined as the solution to $z = we^w$ (see \cref{sec:lambert-w-identities} for more details), and $\rho_0 \defeq \sqrt{1 + e^{\LambW{2/e^2}+2}} \approx 3.19$. We use the temperature scaling \cref{eq:kernel-temperature} in both our theory and experiments. 

\subsection{Maximising throughput via binning}
The sequential nature of \rpnys does not make full use of the massive parallelism available on a GPU. To maximise throughput, we employ a divide-and-conquer strategy that partitions the input into $B$ bins and identifies a coreset of size $r/B$ for each bin in parallel. With this strategy the total operation count is reduced to $O(nr^2/B^2 + nrd/B)$ while the parallel runtime is even faster,  $O(nr^2/B^3 + nrd/B^2)$. 
Our guarantees in \cref{section:theory} account for this binning and allow the user to flexibly trade off between speed and the level of guaranteed accuracy. 

\subsection{Summary of the methodology}
Our core methodology is summarised in \compresskv (\cref{alg:compresskv}) and visualised in \cref{fig:compresskv-visualisation}. After a low-cost recentring and temperature selection using \cref{eq:kernel-temperature}, the keys serve as input data for the \rpnys algorithm with kernel function $h_\tau \defeq \exp(\beta\langle \cdot, \cdot \rangle/\tau^2)$. Using the obtained coreset indices $\mathcal S$ and the Nystr\"om weights $\W \defeq h_\tau(\coreset, \coreset)^{-1} h_\tau(\coreset, \K)$, we form the compressed key and value tensors $\K_{\mathcal S} \defeq (\keys_l)_{l\in \mathcal S}$, $\V_{\mathcal S} \defeq \W \V \in \mathbb R^{r\times d}$. Note that \emph{all} keys and values are involved in the compression of $\V$. In addition, we form the new softmax normalisation vector $\w \defeq \W \boldone_n$ to compute $\hatD_i = \sum_{l\in \mathcal S}\exp(\beta\langle \queries_i, \keys_l\rangle)\w_l$. 
\vspace{-.5\baselineskip}
\begin{algorithm2e}[h]
    \caption{\compresskvnolink}
    \label{alg:compresskv}
    \SetAlgoNoLine\DontPrintSemicolon
    \small
    \begin{algorithmic}
        \STATE {\bf Input:} keys $\K$, values $\V$, radius $\RQ$, scale $\beta$, rank $r$, bins $B$\\[.1\baselineskip]
        \STATE 
        $\bar\keys \gets \K.\texttt{rowsmean()}$; %
        \quad $\K \gets \K -  \bar\keys$ \COMMENT{Recenter keys}
        \STATE Evenly divide (or reshape) rows of $\K$ into bins $\K^1, \dots, \K^B$ %
        \STATE\ForPar{b = 1, \dots, B}{
        \STATE $\RK \gets \max_{l\in [n]} \sqrt{\sum_{j=1}^d(\K^{b}_{l, j})^2}$
        \STATE $\tau \gets \texttt{getTemperature}(\beta, \RQ, \RK, n)$\qtext{using}   \cref{eq:kernel-temperature} \\[.1\baselineskip]
        \STATE $\K^b_{\mathcal S}, \W^b \gets$ $\rpnys(\K^b, \exp(\beta\langle \cdot, \cdot\rangle{/}{\tau^2}), r/B)$
        }
        \STATE // Concatenate (or reshape) bin results %
        \STATE $\coreset \gets (\K^b_\mathcal S)_{b\in [B]}\,;\, \W \gets (\W^b)_{b\in [B]}$
        \STATE $\coreset \gets \coreset + \bar\keys$
        \STATE $\V_\mathcal S \gets \W\V;\,  \w \gets \W\boldone_n$ \COMMENT{Compress values}
        \STATE {\bf Return:} $\coreset, \V_\mathcal S, \w$\!\!\!\!\!\!\!\!\!\!\!\!
    \end{algorithmic}
\end{algorithm2e}
\vspace{-2\baselineskip}
\begin{algorithm2e}[h!]
    \caption{\textsc{WtdAttn}}
    \label{alg:weighted-coreset-attention}
    \small
    \begin{algorithmic}
        \STATE {\bf Input:} queries $\Q$, keys $\K_{\mathcal S}$, values $\V_{\mathcal S}$, weights $\w$,
        \\ \qquad\quad range $(\values_{\min}, \values_{\max})$, scale $\beta$\\[.1\baselineskip]
        \STATE $\widehat\A \gets \exp(\beta \Q \K_{\mathcal S}^\top)$
        \STATE $\widehat{\mathbf O} \gets \mathrm{diag}(\widehat\A\w)^{-1}\widehat\A\V_{\mathcal S}$ $\sstext{\texttt{where}}$ $\widehat\A\w > 0$ $\sstext{\texttt{else}}$ 0
        \STATE {\bf Return:} $\widehat{\mathbf O} \gets \mathrm{clip}(\widehat{\mathbf O}, \values_{\min}, \values_{\max})$
    \end{algorithmic}
\end{algorithm2e}
\vspace{-2\baselineskip}
\begin{algorithm2e}[h!]
    \caption{\cattnolink}
    \label{alg:cmpd-attn}
    \small
    \begin{algorithmic}
        \STATE {\bf Input:} queries $\Q$, keys $\K$, values $\V$, scale $\beta$, rank $r$, bins $B$\\[.1\baselineskip]
        \STATE $(\values_{\min}, \values_{\max}) \gets (\min_{l\in [n]}\V_{l, [d]}, \max_{l\in [n]}\V_{l, [d]})$
        \\[.3\baselineskip]
        \STATE $\RQ \gets \max_{l\in [n]} \sqrt{\sum_{j=1}^d\Q_{l, j}^2}$
        \STATE $\K_{\mathcal S}, \V_{\mathcal S}, \w \gets \compresskv(\K, \V, \RQ, \beta, r, B)$
        \STATE {\bf Return:} $\widehat{\mathbf O} \defeq \wtdattn(\Q, \K_{\mc S}, \V_{\mc S}, \w, \values_{\min}, \values_{\max}, \beta)$ 
    \end{algorithmic}
\end{algorithm2e}
\vspace{-\baselineskip}

In autoregressive models, keys and values of previously processed tokens are stored in KV caches which can incur prohibitive $\Omega(nd)$ memory requirements. In this context (often called the \emph{prefill phase}), we will use \compresskv for \emph{KV cache compression}, requiring only $\bigO(rd)$ in storage for the output.\footnote{While we focus on its memory reduction benefits, KV cache compression also has a complementary computational benefit: $m$ new tokens can be generated in $O(rmd + m^2d)$ time instead of $\Theta(nmd + m^2d)$ time.} The compressed keys and values can then be incorporated into any subsequent attention computations (e.g., to generate new tokens in the \emph{decoding phase}) using the weighted attention forward pass, \wtdattn (\cref{alg:weighted-coreset-attention}).

For non-autoregressive models, we embed \compresskv and \wtdattn into our custom attention module \catt (\cref{alg:cmpd-attn}). 
In the canonical attention approximation setting with $m \sim n$, \catt enjoys $\bigO(nr^2 + nrd)$ runtime, which is near-linear for $r \in n^{o(1)}$.

\section{Approximation Guarantees}\label{section:theory}
We next derive efficient attention approximation guarantees for \catt based on the high quality and low runtime of \rpnys. Recall that \cref{thm:rp-cholesky} allows us to bound the error of \rpnys in terms of any benchmark approximation $\T\preceq \Hc_{\tau}$. To obtain a concrete bound, we consider $\T^s$ induced by an order $s$ Taylor approximation of the exponential function:
\begin{talign}
\mathbf T_{il}^s \defeq \sum_{p = 0}^s \frac{1}{p!}\big(\frac{\beta}{\tau^2}\langle \keys_i, \keys_l\rangle\big)^{p}.
\end{talign}
Our next lemma, proved in \cref{subsection:taylor-bound}, characterises the trade-off between order and approximation accuracy.
\begin{lemma}[Taylor  guarantee]\label{cor:trace-polynomial-guarantees}
Define the order parameter
\begin{align}
    \tilde s(\varepsilon) \defeq \frac{\log(n/\varepsilon) + \beta \RK^2/\tau^{2}}{\LambW*{\frac{\log(n/\varepsilon)\tau^{2}}{e\beta\RK^2}+ \frac{1}{e}}}
\end{align}
for $\eps>0$ where $W_0$ is the primary branch of the Lambert-W function.
Then, $\nuc{\Hc_\tau - \T^s} \leq \varepsilon$ for all $s \geq \lfloor \tilde s(\varepsilon)\rfloor$.
\end{lemma}

Meanwhile, \cref{lem:rank-entropy}, proved in \cref{proof-lem:binom}, bounds the rank of $\T^s$ in terms of its order. %
\begin{lemma}[Taylor rank bound]\label{lem:rank-entropy}
    For any $s\in\naturals$, 
    \begin{talign}
        \mathrm{rank}(\T^s) \!\leq \!\frac{1}{\sqrt{\pi}} n^{(\sigma + \delta)\mathrm{Ent}\paren{\frac{\sigma}{\sigma + \delta}}}\stext{for} (\sigma,\delta) \!\defeq \!(\frac{s}{\log(n)}, \frac{d}{\log(n)})
    \end{talign}
    and  $\mathrm{Ent}(p) \defeq -p\log(p) - (1-p)\log(1-p)$.
\end{lemma}
Combining \cref{prop:transformer-bound-via-A-approximation,prop:nystrom,cor:trace-polynomial-guarantees,lem:rank-entropy} with \cref{thm:rp-cholesky}, we arrive at the following attention approximation guarantee for \catt (proved in \cref{subsection:proof:thm:cmpd-attn-guarantees}). %
\begin{theorem}[\cattnolink guarantee]\label{thm:cmpd-attn-guarantees-non-aymptotic}
Let $\widehat {\mathbf O}_r$ be the output of \catt (\cref{alg:cmpd-attn}) with rank parameter $r$ and $B=1$. 
Fix $a\geq \frac{1}{2}$ and define the entry and dimension growth parameters,
    \begin{talign}
        \gamma \defeq \frac{\beta\RQ\RK}{\log(n)} \qtext{and}
        \delta \defeq \frac{d}{\log(n)},
    \end{talign}
    along with the Taylor growth parameter
    \begin{talign}\label{eq:taylor-order-explicit}
        \sigma \defeq 
        \frac{a + \gamma}{\LambW*{\frac{1}{2\rho_0 \gamma} + \frac{1}{\rho_0}}}\,.
    \end{talign}
    Then, \scalebox{0.9}{$\mathbb E\error[r] \leq 3\Vnorm n^{-a}$}
    provided
    \begin{talign}\label{eq:wildcat-r-bound}
        r \geq 1 + \frac{1}{\sqrt{\pi}}n^{(\sigma + \delta)\mathrm{Ent}\paren{\frac{\sigma}{\sigma + \delta}}}\log\paren{n^{2a + \sigma + 3\gamma}}.
    \end{talign}
    For $B > 1$, the same result holds with the effective sequence length and rank $(n_{\mathrm{eff}},r_{\mathrm{eff}})  = (\lfloor \frac{n}{B}\rfloor, \lceil \frac{r}{B}\rceil)$ in place of $(n,r)$.
\end{theorem}
\cref{thm:cmpd-attn-guarantees-non-aymptotic} lets us easily identify conditions under which \catt guarantees super-polynomial accuracy in near-linear time. For example, in \cref{tab:comparison-approximate-attention-guarantees}, we compare the error decay guarantees of \cref{thm:cmpd-attn-guarantees-non-aymptotic} with those of various  practical attention approximations %
assuming bounded dimension, bounded entries, $m=n$, and $\bigO(dn^{1+t})$ runtime. Notably, Thinformer \citep{carrell2025low}, BalanceKV \citep{han2025streaming}, KDEformer \citep{zandieh2023kdeformer}, and HyperAttention without masking \citep{han2024hyperattention} all guarantee at best polynomial error decay, while \catt provides \emph{super}-polynomial $\bigO(n^{-\Omega(\log(\log(n)))t})$ error decay.

\begin{table}[tb]
    \centering
    \caption{\tbf{Practical approximation guarantees.}
    For each approximation $\Ohat$ to the softmax matrix $\O$ \cref{eq:attn-output} with $m=n$, we report, up to constants, the best worst-case error bound on $\error$ given 
    bounded dimension $d \in \bigO(1)$, 
    bounded entries $\beta\RQ^2,\beta\RK^2 \leq R^2\in O(1)$, 
    and $\bigO(dn^{1+t})$ runtime. Here, the ratios   $\opnorm{\V}/\maxnorm{\V}$ and $\fronorm{\V}/\maxnorm{\V}$
    lie in $[1,\sqrt{nd}]$, $\xi \defeq 0.173+o(1)$, and $\kappa \defeq e^{-1}(2\rho_0 + 1)$. 
    See \cref{proof-tab:comparison-approximate-attention-guarantees} for the proof of each guarantee.}\label{tab:comparison-approximate-attention-guarantees}
    \begin{tabular}{cc}
        \toprule
        \Centerstack{\bf Approximation}
         & \Centerstack{\bf Guarantee}
         \\\midrule\\[\tabletopskip] 
         \Centerstack{\bf Thinformer} 
         & %
         $\frac{\sqrt{\log( \smaxnorm{\V})}\log n}{n^{t}}\cdot\rownorm{\V}$ %
         \\[\tablelineskip]
         
          \Centerstack{\bf BalanceKV} 
         & %
         $\frac{(\log n)^3}{n^{t}}\cdot\fronorm{\V}$ %
         \\[\tablelineskip]        \Centerstack{\textbf{KDEformer}}%
         & %
         $\frac{n^{\xi/2}}{n^{t/2}}\cdot\specnorm{\V}$ %
          \\[\tablelineskip]
         \Centerstack{\textbf{HyperAttention}}%
         & $\frac{(\log n)^{{1/}{6}}}{n^{t/6}}\cdot\specnorm{\V}$%
         \\[\tablelineskip]
         \Centerstack{\textbf{\cattnolink}}%
         & $ \frac{\log n}{n^{0.14 t\log\paren{e + \log(n)/(\kappa R)}}}\cdot\maxnorm{\V}$
         \\[2mm]
         \bottomrule
    \end{tabular}%
\end{table}

Plugging in $t = 8/\sqrt{\log\log(n)}$, we observe that \catt can even deliver super-polynomial $O(n^{-\sqrt{\log \log(n)}})$ error decay in near-linear $O(dn^{1+o(1)})$ time. In fact, this remains true even when the entries and dimension are allowed to grow with the sequence length:
\begin{corollary}[Super-polynomial error decay in near-linear time]\label{cor:simple-superpoly}
Under the assumptions of \cref{thm:cmpd-attn-guarantees-non-aymptotic}, suppose $\beta\RQ\RK\in \bigO(\log(n)^\alpha)$ with $\alpha \in (0, 1)$,  $d\in o(\log(n))$, and $a(n) \in o(\log\log(n))$.
Then
\begin{talign}
\mathbb E\error[r] \leq 3\Vnorm n^{-a(n)}
\stext{for some}
r\in n^{o(1)}.
\end{talign}
\end{corollary}
\begin{proof}
    For $\beta\RQ\RK\in \bigO(\log(n)^\alpha)$ and $a(n) \in o(\log(\log(n)))$, a short calculation (see \cref{prop:asymptotic-behaviour-taylor-order-parameter}) shows that there exists a $c>0$ and $n_0 > 0$ such that $\sigma(n) \leq c\frac{a(n)}{\log\paren{1 + \gamma(n)^{-1}}} \leq c \frac{a(n)}{\alpha\log(\log(n))} \in o(1)$ for all $n>n_0$. Furthermore, by definition, $\delta = d/\log(n)\in o(1)$. Therefore, we obtain from \cref{thm:cmpd-attn-guarantees-non-aymptotic} that it is sufficient to take $r \sim n^{(\sigma + \delta)\mathrm{Ent}\paren{\frac{\sigma}{\sigma + \delta}}}\log\paren{n^{2a + \sigma + 3\gamma}} \in n^{o(1)}$ to guarantee $n^{-a(n)}$ error decay.
\end{proof}
Meanwhile, the HyperAttention, KDEformer, and Thinformer guarantees 
deliver, at best, near-constant $n^{-o(1)}$ error in near-linear time and  require quadratic time to guarantee super-polynomial error decay. 

While we followed \citet{alman2023fast,keles2023computational,carrell2025low} in stating entrywise error guarantees, our results also improve upon the operator norm guarantees established for KDEformer and Hyperattention and the $\rownorm{\cdot}$ guarantees established for BalanceKV. Indeed, even using  the lossy conversions  $\lVert\mathbf O - \widehat{\mathbf O}\rVert_{\mathrm{op}} \leq \sqrt{nd} \error$ and $\lVert \mathbf O - \widehat {\mathbf O}\rVert_{2,\infty} \leq \sqrt{d} \error$, our guarantees for $\error$ imply super-polynomial decay in near-linear time for the other norms. Meanwhile, the guarantees for prior work remain exactly as in \cref{tab:comparison-approximate-attention-guarantees} with each requiring quadratic time for super-polynomial error decay and achieving at best sub-polynomial, near-constant decay in near-linear time.

Pushing beyond the limits of prior work on the computational hardness of attention \citep{alman2023fast,keles2023computational}, our next corollary shows that \catt can achieve super-polynomial error decay in near-linear time even when the dimension or entries grow super-logarithmically in $n$. 
Our proof in \cref{subsection:derived-asymptotic results} uses the entropy factor in \cref{thm:cmpd-attn-guarantees-non-aymptotic} to refine the analysis in \cref{cor:simple-superpoly}.
\begin{corollary}[Refined super-polynomial error decay in near-linear time]\label{cor:guarantee-cmpd-attn-asymptotic}
Instantiate the assumptions of \cref{thm:cmpd-attn-guarantees-non-aymptotic}. 
If $\gamma(n) \defeq \frac{\beta\RQ\RK}{\log(n)} \in o(1)$, $\delta(n) \defeq \frac{d}{\log(n)}$, and
\begin{talign}
a(n) \in o\paren{\frac{\log(1/\gamma(n))}{\max\{\log(\delta(n)), 1\}}}\cap n^{o(1)},
\end{talign}
then $\mathbb E\error[r] \leq 3\Vnorm n^{-a(n)}$
    for some
    $r \in n^{o(1)}$.
The same conclusion holds if, alternatively, $\delta(n) \in o(1)$, $\gamma(n) \in \Omega(1)\cap n^{o(1/d)}$, and $a(n) \in n^{o(1/d)}$.
\end{corollary}

Let us consider two important implications of \cref{cor:guarantee-cmpd-attn-asymptotic}.  First, when $d$ is bounded (the typical case when one is working with a fixed model and focused on increasing its context length), \cref{cor:guarantee-cmpd-attn-asymptotic} supports unbounded entries with \emph{any} form of near-constant $\beta\RQ\RK\in n^{o(1)}$ growth, even super-logarithmic $\omega(\log(n))$ growth. In contrast, the near-linear-time theory of \citet[Thm.~3.8]{alman2023fast} only guarantees polynomial error decay for $\beta\RQ\RK\in o(\log n)$.

Second, \cref{cor:guarantee-cmpd-attn-asymptotic} also supports super-logarithmic dimension growth, $d=\omega(\log n)$. For example, when the scaled entries $\beta\RQ\RK$ are bounded, any choice of $d = \log(n)\exp(o(\log(\log n)))$ still leads to super-polynomial decay in near-linear time. In contrast, the near-linear-time theory of \citet{alman2023fast} only guarantees polynomial error decay for $d\in O(\log n)$. Interestingly,  \citet{keles2023computational} also provide a \emph{lower} bound on the speed of attention approximation when $d\in\omega(\log n)$.
Assuming the strong exponential-time hypothesis (a widely-believed conjecture in complexity theory), \citet[Thm.~6]{keles2023computational} shows that, for any $\eps>0$, approximating attention to absolute error $n^{-3d}e^{-3d^2}$ with $d=\omega(\log n)$ requires $\Omega(n^{2-\eps})$ time. However, this lower bound still allows for super-polynomial error decay in near-linear time like that established by \cref{cor:guarantee-cmpd-attn-asymptotic}.

\textbf{Guarantees for KV cache compression:}
The \compresskv methodology enables us to reduce the memory footprint of KV caches from linear $\Theta(nd)$ to near-constant $\mathcal O(rd)$ whenever $r \in n^{o(1)}$. Long-context inference via \wtdattn with the compressed cache $(\K_{\mc S}, \V_{\mc S}, \w)$ then still obeys the accuracy guarantees of \cref{thm:cmpd-attn-guarantees-non-aymptotic} and \cref{cor:guarantee-cmpd-attn-asymptotic}.

\subsection{Additional related work}

Two other fast attention methods use the Nystr\"om method in their methodology: \citet{xiong2021nystromformer} approximate the matrix $\Dinv\A$ with a Nystr\"om approximation, directly. The evaluation of any entry in $\Dinv\A$ requires the realisation of $m\times n$ entries. Nystr\"omformer therefore requires additional sketches of the softmax matrix and does not offer strong guarantees. \citet{chen2021skyformer}, on the other hand, use the Nystr\"om method to approximate a distinct Gaussian attention mechanism. Finally, \citet{chen2022sketching} uses sketching methods to construct low-rank approximations of the attention mechanism.
\begin{figure*}[!t]
\centering
\begin{minipage}{\textwidth}
    \begin{minipage}{\samplempwid\textwidth}
    \centering
    Exact
    \includegraphics[width=\textwidth]{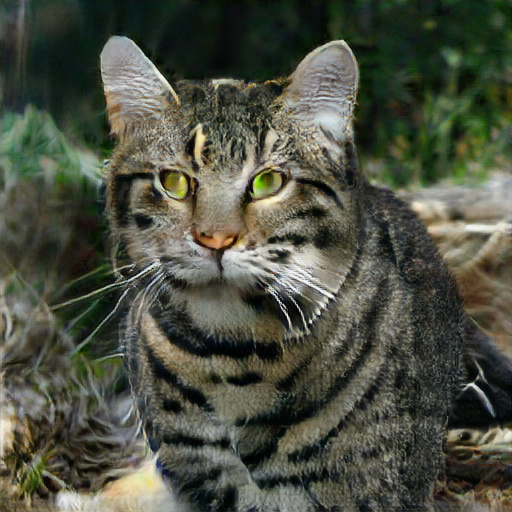}
    \end{minipage}
\hspace{-\samplehinterval}
    \begin{minipage}{\samplempwid\textwidth}
    \centering
    \textsc{WildCat} %
    \includegraphics[width=\textwidth]{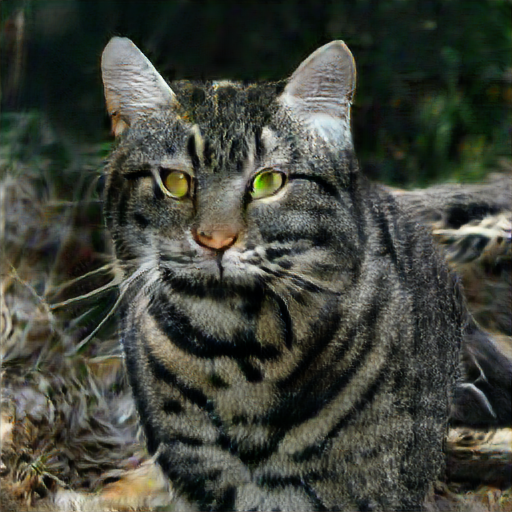}
    \end{minipage}
\hspace{-\samplehinterval}
    \begin{minipage}{\samplempwid\textwidth}
    \centering
    Performer %
    \includegraphics[width=\textwidth]{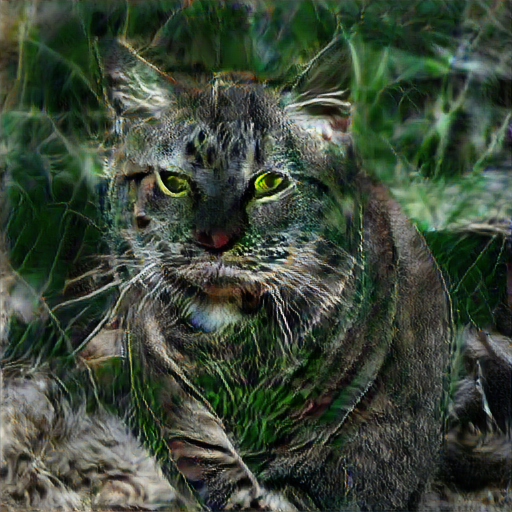}
    \end{minipage}
\hspace{-\samplehinterval}
    \begin{minipage}{\samplempwid\textwidth}
    \centering
    Reformer %
    \includegraphics[width=\textwidth]{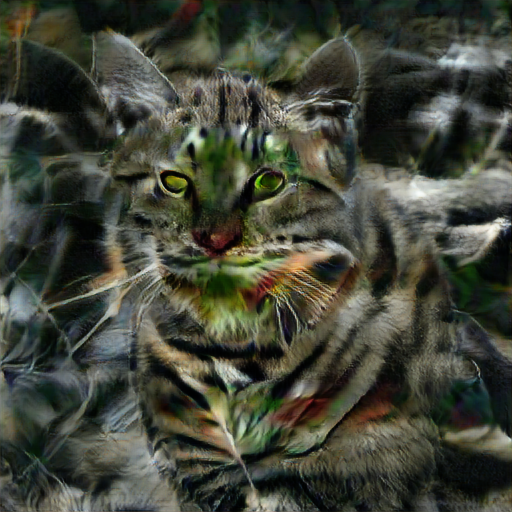}
    \end{minipage}
\hspace{-\samplehinterval}
    \begin{minipage}{\samplempwid\textwidth}
    \centering
    KDEformer %
    \includegraphics[width=\textwidth]{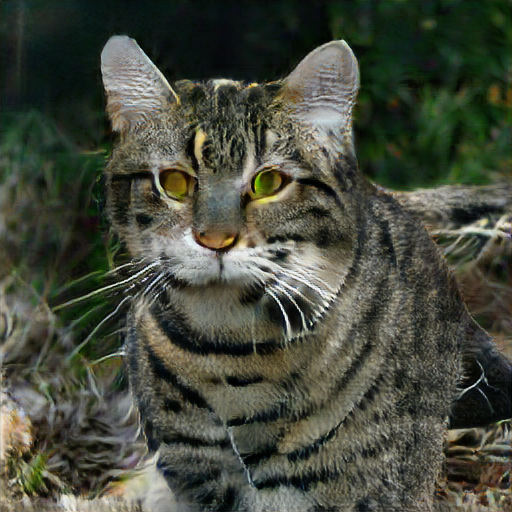}
    \end{minipage}
\hspace{-\samplehinterval}
    \begin{minipage}{\samplempwid\textwidth}
    \centering
    Thinformer %
    \includegraphics[width=\textwidth]{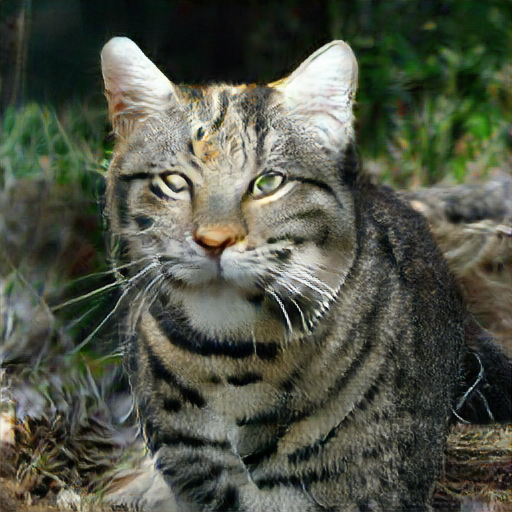}
    \end{minipage}
\hspace{-\samplehinterval}
\end{minipage}

\centering
\caption{\textbf{Example generations from BigGAN with exact or approximate attention.}
}
\vspace{-2mm}
\label{fig:biggan_result_visualisation}
\end{figure*}
\begin{table*}[t]
\caption{\tbf{Quality of attention approximations for BigGAN image generation.} 
We report speed-ups over $10$ batches of $32$ images and mean degradation ($\pm 1$ standard deviation across five seeds) of the Inception Score (IS) and Frechet Inception Distance (FID) between $5$K generations and the ImageNet 2012 validation set.}%
    \label{tab:biggan}
\small    
\centering

\begin{tabular}{cccc}
\toprule
\textbf{Attention Algorithm} & \textbf{Speed-up over Exact} & \textbf{IS Degradation (\%)} & \textbf{FID Degradation (\%)} \\
\midrule
\textbf{Reformer} & $0.69\times$ & $66.55 \pm 0.52$ & $124.20 \pm 1.19$ \\[1mm]
\textbf{ScatterBrain} & $1.75\times$ & $36.77 \pm 0.50$ & $20.87 \pm 1.25$ \\[1mm]
\textbf{Performer} & $3.46\times$ & $35.14 \pm 0.82$ & $4.01 \pm 0.91$ \\[1mm]
\textbf{KDEformer} & $0.72\times$ & $2.02 \pm 0.81$ & $\mathbf{0.00} \pm 0.00$ \\[1mm]
\textbf{Thinformer} & $2.32\times$ & $1.79 \pm 0.31$ & $\mathbf{0.00} \pm 0.00$ \\[1mm]
\textbf{\textsc{WildCat}} & $\mathbf{4.33\times}$ & $\mathbf{1.22} \pm 0.87$ & $\mathbf{0.00} \pm 0.00$ \\
\bottomrule
\end{tabular}
\end{table*}

Our method relies on the reduction of a dataset to a representative weighted coreset with similarity of datapoints measured by the attention kernel. From this standpoint, it falls into the class of \emph{distribution compression} methods which aim to succinctly summarize an empirical or population distribution using a small collection of representative points and into the class of \emph{kernel quadrature} methods which aim to accurately approximate expectations of functions in an reproducing kernel Hilbert space \citep{aronszajn1950theory}. 
Early works that use kernels for distribution compression with unweighted coresets include \citet{dwivedi2024kernel, dwivedigeneralized,shetty2022distribution,gong2024supervised,carrell2025low}. Additional strategies to achieve compression with guarantees for weighted coresets were proposed in \citet{hayakawa2022positively, epperly2023kernel, li2024debiased}.  A method to approximate ratios of kernel sums using the \Nystrom method was explored empirically in \citet{gong2024supervised} but without providing guarantees.
\section{Experiments}
\label{sec:experiments}
We now turn to an empirical evaluation of our new tools on a suite of standard approximate attention benchmarks. See
\begin{align}
\text{\url{https://github.com/microsoft/wildcat}}
\end{align}
for open-source PyTorch \citep{paszke2019pytorch} code recreating all experiments and \cref{app:experiment-details} for supplementary experiment details.
\begin{table*}[t]
\caption{\tbf{Quality of attention approximations for T2T-ViT ImageNet classification.}
We report speed-ups over $50$ batches of $64$ images and mean Top-$1$ accuracy $\pm1$ standard deviation across five seeds.}%
    \label{tab:t2t}
    \small
    \centering
\begin{tabular}{cccc}
\toprule
\textbf{Attention Algorithm} & \textbf{Top-1 Accuracy (\%)} & \textbf{Layer 1 Speed-up} & \textbf{Layer 2 Speed-up} \\
\midrule
\textbf{Exact} & $82.55 \pm 0.00$ & $1.00\times$ & $1.00\times$ \\[0.5mm]
\textbf{Performer} & $80.91 \pm 0.18$ & $7.29\times$ & $1.82\times$ \\[0.5mm]
\textbf{Reformer} & $81.47 \pm 0.06$ & $2.35\times$ & $0.92\times$ \\[0.5mm]
\textbf{KDEformer} & $82.04 \pm 0.02$ & $3.28\times$ & $0.49\times$ \\[0.5mm]
\textbf{ScatterBrain} & $82.05 \pm 0.02$ & $2.65\times$ & $0.77\times$ \\[0.5mm]
\textbf{Thinformer} & $82.16 \pm 0.02$ & $8.84\times$ & $2.61\times$ \\[0.5mm]
\textbf{\textsc{WildCat}} & $82.19 \pm 0.04$ & $11.59\times$ & $2.65\times$ \\
\bottomrule
\end{tabular}
\end{table*}
\begin{table*}[!t]
\caption{\label{tab:longbenche}
\textbf{Quality of KV cache compression for LongBench-E long-context language understanding.}
}
\centering
\small
\begin{tabular}{@{\hspace{2.5pt}}c@{\hspace{4pt}}c@{\hspace{4pt}}c@{\hspace{4pt}}c@{\hspace{4pt}}c@{\hspace{4pt}}c@{\hspace{4pt}}c@{\hspace{4pt}}c@{\hspace{4pt}}c@{\hspace{4pt}}c@{\hspace{4pt}}c@{\hspace{4pt}}c@{\hspace{4pt}}c@{\hspace{4pt}}c@{\hspace{4pt}}c@{\hspace{2.5pt}}}
\toprule
\textbf{Method} & \textbf{qasper} & \textbf{multifield} & \textbf{hotpot} & \textbf{2wiki} & \textbf{gov} & \textbf{multinews} & \textbf{trec} & \textbf{trivia} & \textbf{samsum} & \textbf{p.count} & \textbf{p.ret} & \textbf{lcc} & \textbf{repo-p} & \textbf{average} \\
\midrule
\multicolumn{15}{c}{\textbf{75.0\% Compression}} \\
\midrule
\textbf{Exact} & 43.76 & 50.08 & 56.43 & 43.71 & 34.16 & 24.32 & 64.00 & 87.53 & 38.56 & 14.67 & 99.67 & 69.96 & 62.17 & 53.00 \\[.5mm]
\textbf{StreamingLLM} & 23.17 & 25.49 & 26.45 & 20.89 & 29.65 & 22.25 & 52.33 & 75.17 & 35.79 & 12.50 & 24.83 & 68.48 & 56.22 & 36.40 \\[.5mm]
\textbf{PyramidKV} & 21.59 & 29.96 & 39.38 & 30.24 & 27.12 & 21.13 & 43.33 & 86.77 & 38.27 & 15.88 & 61.06 & 67.40 & 59.20 & 41.64 \\[.5mm]
\textbf{BalanceKV} & 29.50 & 36.57 & 37.89 & 23.71 & 30.27 & 21.98 & 55.00 & 73.83 & 34.56 & 12.67 & 71.67 & 65.48 & 62.57 & 42.75 \\[.5mm]
\textbf{Uniform} & 26.91 & 37.51 & 38.46 & 25.93 & 30.02 & 21.86 & 54.33 & 81.71 & 35.37 & 15.33 & 63.44 & 64.84 & 61.34 & 42.85 \\[.5mm]
\textbf{SnapKV} & 25.52 & 30.13 & 44.36 & 31.80 & 29.70 & 22.10 & 49.33 & 88.32 & 37.15 & 16.67 & 89.06 & 69.16 & 56.33 & 45.36 \\[.5mm]
\textbf{CompressKV} & 33.23 & 38.13 & 43.43 & 33.37 & 30.30 & 22.26 & 54.33 & 86.15 & 35.38 & 14.33 & 98.00 & 64.85 & 60.43 & \textbf{47.25} \\[.5mm]
\midrule
\multicolumn{15}{c}{\textbf{87.5\% Compression}} \\
\midrule
\textbf{Exact} & 43.76 & 50.08 & 56.43 & 43.71 & 34.16 & 24.32 & 64.00 & 87.53 & 38.56 & 14.67 & 99.67 & 69.96 & 62.17 & 53.00 \\[.5mm]
\textbf{StreamingLLM} & 19.27 & 24.62 & 24.34 & 21.55 & 26.40 & 20.49 & 47.33 & 71.71 & 33.61 & 10.67 & 15.67 & 67.24 & 59.13 & 34.00 \\[.5mm]
\textbf{PyramidKV} & 17.69 & 24.43 & 31.94 & 26.08 & 25.48 & 20.34 & 40.33 & 87.41 & 37.68 & 13.89 & 42.22 & 66.61 & 58.74 & 37.91 \\[.5mm]
\textbf{BalanceKV} & 17.90 & 28.87 & 27.79 & 17.94 & 27.45 & 20.78 & 45.67 & 62.84 & 33.29 & 10.67 & 32.22 & 60.95 & 60.90 & 34.41 \\[.5mm]
\textbf{Uniform} & 16.79 & 30.22 & 27.98 & 18.50 & 27.10 & 20.90 & 44.67 & 68.40 & 33.95 & 13.00 & 26.00 & 62.08 & 59.91 & 34.58 \\[.5mm]
\textbf{SnapKV} & 16.36 & 25.74 & 35.15 & 24.96 & 26.38 & 20.76 & 45.17 & 88.29 & 37.09 & 14.00 & 58.11 & 68.84 & 56.41 & 39.79 \\[.5mm]
\textbf{CompressKV} & 23.16 & 30.14 & 35.27 & 24.96 & 27.70 & 21.16 & 42.33 & 83.23 & 34.33 & 14.67 & 87.06 & 63.46 & 60.05 & \textbf{42.12} \\[.5mm]
\midrule
\multicolumn{15}{c}{\textbf{93.75\% Compression}} \\
\midrule
\textbf{Exact} & 43.76 & 50.08 & 56.43 & 43.71 & 34.16 & 24.32 & 64.00 & 87.53 & 38.56 & 14.67 & 99.67 & 69.96 & 62.17 & 53.00 \\[.5mm]
\textbf{StreamingLLM} & 13.98 & 22.17 & 23.05 & 21.25 & 23.27 & 18.34 & 38.33 & 65.88 & 31.82 & 6.67 & 8.11 & 63.30 & 56.67 & 30.22 \\[.5mm]
\textbf{PyramidKV} & 11.46 & 22.21 & 30.97 & 23.14 & 22.94 & 18.86 & 32.00 & 85.48 & 36.66 & 11.67 & 18.78 & 64.06 & 56.46 & 33.44 \\[.5mm]
\textbf{BalanceKV} & 10.09 & 24.12 & 19.80 & 18.34 & 24.17 & 19.62 & 24.33 & 51.58 & 32.23 & 8.67 & 5.33 & 58.16 & 59.15 & 27.35 \\[.5mm]
\textbf{Uniform} & 11.24 & 24.06 & 21.17 & 16.70 & 24.22 & 19.38 & 31.83 & 54.58 & 32.97 & 11.67 & 7.00 & 55.81 & 58.48 & 28.39 \\[.5mm]
\textbf{SnapKV} & 11.47 & 22.52 & 30.79 & 23.13 & 23.18 & 18.89 & 31.33 & 86.09 & 36.63 & 11.67 & 18.78 & 64.51 & 55.67 & 33.44 \\[.5mm]
\textbf{CompressKV} & 15.03 & 24.71 & 28.99 & 24.01 & 24.63 & 19.51 & 26.00 & 80.27 & 34.13 & 13.00 & 62.33 & 62.28 & 60.26 & \textbf{36.55} \\[.5mm]
\bottomrule
\end{tabular}
\end{table*}
\subsection{Benchmarking image generation}\label{sec:biggan}
We begin with the BigGAN image generation benchmark of \citet{carrell2025low}.
BigGAN \citep{brocklarge} is a generative adversarial network for image generation containing a single attention layer, which, for images of size $512\times 512$, has input tensors $\Q \in \mathbb R^{4096\times 64}, \K \in \mathbb R^{1024\times 64},$ and $ \V\in \mathbb R^{1024\times 256}$. 
The BigGAN benchmark evaluates the quality of attention approximations used as drop-in replacements for exact attention in a BigGAN model pretrained on ImageNet \citep{deng2009imagenet}. 
Using the settings and implementations provided by \citet{carrell2025low}, we benchmark \catt (with $r=96$ and $B=8$) against exact attention and five leading approximate attention mechanisms: Reformer, ScatterBrain, Performer, KDEformer, and Thinformer.
\cref{fig:biggan_result_visualisation} displays example generations, and \cref{tab:biggan} reports the gain in speed and the loss in  quality from using each approximation to generate $5000$ images. We observe that \catt yields the largest speed-up ($4.33\times$), the smallest degradation in Inception Score \citep[IS,][]{salimans2016improved} (just $1.22\%$), and, surprisingly, no degradation in Frechet Inception Distance \citep[FID,][]{heusel2017gans}.  

\subsection{Benchmarking image classification}\label{sec:t2t}
We next replicate the  Tokens-to-Token Vision Transformer (T2T-ViT) image classification benchmark of \citet{carrell2025low}, where attention approximations are used as drop-in replacements for exact attention in the computationally demanding tokens-to-token module. T2T-ViT \citep{yuan2021tokens}
splits an input image into a large number of overlapping patches which are progressively reduced to a smaller number of tokens by two attention layers. 
The T2T-ViT benchmark uses a model pretrained on ImageNet with images of size $224\times 224$, layers of size $(n_1,d_1) = (3136,64)$ and $(n_2,d_2) = (784,64)$, and a computational cost dominated by the larger first layer.

In \cref{tab:t2t}, we benchmark \catt, with $(r_1,B_1)=(224,224)$ for the first layer and $(r_2,B_2) = (196,196)$ for the second, against exact attention and the five leading approximate attention mechanisms of \cref{sec:biggan} using the settings and implementations provided by \citet{carrell2025low}. 
Amongst the approximations, \catt provides the highest mean Top-$1$ accuracy ($82.19\%$ vs. $82.55\%$ for exact) while also yielding the lowest runtime for each layer, including an $11.59\times$ speed-up for the dominant layer 1.

\subsection{Benchmarking KV cache compression}\label{sec:kvcache}
Finally, we evaluate the performance of  \compresskv on $13$ benchmark KV cache compression tasks with the Qwen2.5-7B-Instruct language model \citep{yang2024qwen2}. In transformer-based autoregressive generative models, only the queries, keys, and values associated with the last decoded token of the sequence %
have to be computed from hidden states, while $(\keys_l, \values_l)$ for $l < n$ can be stored in a cache to avoid recomputation. However, as the context length $n$ increases, the KV cache memory eventually becomes a bottleneck, limiting the maximum number of past tokens that can be considered during inference. 
KV cache compressors conserve memory by extracting a smaller set of $r$ keys and values from the context and attending only over those $r$ context pairs during generation.

We begin with an empirical verificiation of the assumptions underlying our strongest compression guarantees. 
Specifically, we test the assumptions of \cref{cor:guarantee-cmpd-attn-asymptotic} using Qwen2.5-7B-Instruct and document-grounded question answering sequences from the QASPER-E dataset \citep{bai2024longbench}. Since $d$ is constant for any fixed model, \cref{cor:guarantee-cmpd-attn-asymptotic} applies as long as $\beta \RQ\RK \in n^{o(1)}$. In \cref{tab:gamma-empirical}, we find that $\gamma(n) = \tfrac{\beta \RQ\RK}{\log(n)}$, averaged all layers and the first $10$ sequences with $n\geq 16384$, is not only bounded but is in fact decreasing with $n$. By \cref{cor:guarantee-cmpd-attn-asymptotic}, \compresskv can therefore approximate attention with super-polynomially decaying error using a near-constant cache size $r$ for this model and task. 
Interestingly, this concordance with our assumptions is also implied by the work of \citet[proof of Thm.~2.2]{velivckovic2025softmax}, who showed that \emph{any} fixed transformer-based model with a finite vocabulary has all query and key norms bounded independently of the sequence length $n$.
\begin{table}[t!]
\caption{For document-grounded question answering with Qwen2.5-7B-Instruct, the entry growth factor $\gamma(n) = \tfrac{\beta \RQ\RK}{\log(n)}$ of \cref{cor:guarantee-cmpd-attn-asymptotic} decreases as a function of context length $n$.}
\label{tab:gamma-empirical}
\centering
\small
\begin{tabular}{cccccccc}
\toprule
$\hspace{-0.1cm}\boldsymbol{n}$\!\! &4 & 16 & 64 & 256 & 1024 & 4096 & 16384 \!\!\!\\
\midrule
$\hspace{-0.1cm}\boldsymbol{\gamma(n)}$\!\! & 14.95 & 9.55 & 7.48 & 6.70 & 6.23 & 5.86 & 5.63 \!\!\!\\
\bottomrule
\end{tabular}
\vskip -0.1in
\end{table}

Following the experimental setup of \citet{han2025streaming}, we next benchmark compression quality using the $13$  LongBench-E tasks of \citet{bai2024longbench}. 
These diverse tasks cover a wide range of long-context language understanding applications including single and multi-document question answering, summarisation, few-shot learning, and code completion. 
Following \citet{han2025streaming}, we compress each cache by $75\%$, $87.75\%$, and $93.75\%$ 
and benchmark \compresskv with $B=\frac{r}{12}$ against no cache compression (``Exact'') and five leading cache compression methods: 
StreamingLLM \citep{xiaoefficient}, PyramidKV \citep{cai2024pyramidkv}, BalanceKV and Uniform \citep{han2025streaming}, and SnapKV \citep{li2024snapkv}. We use the implementations of \citet{han2025streaming} for BalanceKV and Uniform and those provided by KVPress \citep{devoto2025expected} for the remaining methods. 
As in \citet{han2025streaming}, BalanceKV, Uniform, and \compresskv all retain the first and last $32$ context tokens and compress the remaining tokens to achieve the desired compression level. 

\cref{tab:longbenche} reports a standard measure of compression quality for each LongBench-E task as well as the average compression quality across all $13$ tasks. 
Remarkably, for each compression level, \compresskv yields the highest average compression quality and the highest individual task quality on a plurality of the $13$ tasks. %

\subsection{Benchmarking against FlashAttention 2}\label{flash}
We additionally benchmark \catt attention with $r=64$ and $B=16$ against the highly-optimized, I/O-aware FlashAttention 2 \citep[FA2,][]{dao2024flashattention} implementation of exact attention using $(\Q, \K, \V)$ inputs with $d=64$, $n$ ranging from $2^{13}$ to $2^{18}$, and independent standard Gaussian entries. As the sequence length increases, we observe in \cref{fig:flash} both a steady increase in speed-up over FA2 (from $1.1\times$ to $68\times$) and a steady decrease in approximation error $\error[]$. Additional ablations over the $r$ and $B$ parameters can be found in \cref{supp_flash}.

\begin{figure}
    \centering
    \includegraphics[width=\linewidth]{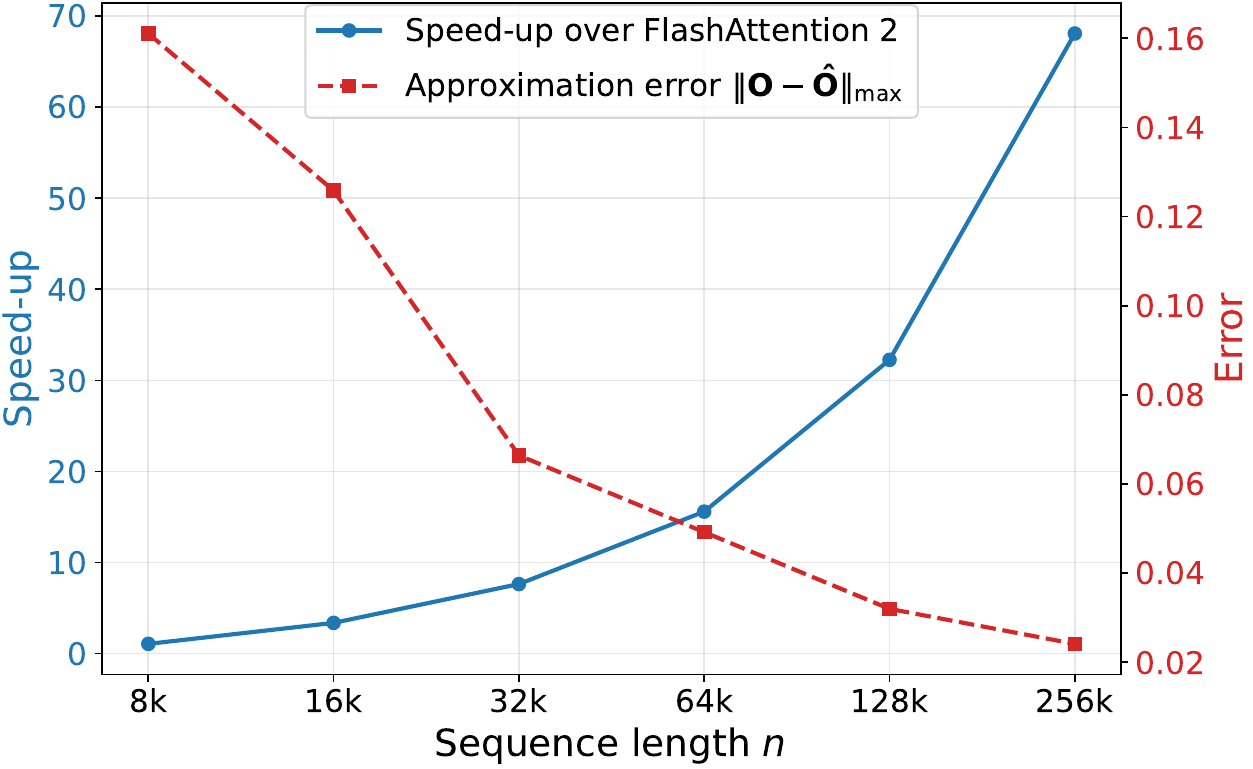}
    \caption{\catt approximation error and speed-up over Flash Attention 2. See \cref{flash} for more details.}
    \label{fig:flash}
\end{figure}
\section{Conclusions}
We introduced \catt, a principled method for cheaply, accurately, and practically approximating  softmax attention. \catt efficiently distills the information from all keys and values into a small coreset optimally weighted for attention reconstruction. Our fast but spectrally-accurate subsampling procedure and optimised weighting allow us to achieve super-polynomially decaying error guarantees while maintaining near-linear runtime and near constant KV cache sizes. To bridge the gap between theory and practice, we additionally developed a GPU-optimised PyTorch implementation and demonstrated the practical benefits of \catt for image generation, image classification, and KV cache compression.
That said, this work is not without its limitations. For example, this work does not address the important problem of streaming generation with causal masking, and we suspect such an extension is possible via prefix sums as in \citet{choromanskirethinking} or divide-and-conquer evaluations as in \citet{han2024hyperattention}. 
In future work, we also aim to address the path dependence and sequential nature of the pivot selection process in \rpnys. This could be achieved by oversampling mechanisms, such as those used in accelerated RPCholesky \citep{epperly2024embrace} and recursive leverage-score sampling methods \citep{musco2017recursive, rudi2018fast}. Fortunately, our modular analysis does allow for us to substitute any of these fast, spectrally accurate subsampling algorithms for \rpnys.  However, these alternative procedures  adaptively adjust coreset sizes across iterations, introducing additional challenges for batch-wise processing.

Finally, while our methodology naturally extends to other forms of kernelised attention \citep{tsai2019transformer}, more work may be required to extend our analysis. Such extensions are possible with sufficient knowledge of the spectral properties of the kernel data matrix. One set of tools for this purpose consists of sampling inequalities that express the low-rank approximability of a kernel matrix in terms of the fill distance of the underlying domain \citep{narcowich2005sobolev, RiegerSampling, fuselier2012scattered, altschuler2019massively}. We suspect such tools will also yield improved runtime and error guarantees under additional smoothness or manifold \citep{zhu2018ldmnet} assumptions on the attention inputs.

\section*{Impact Statement}
By improving the trade-off between resource consumption and model quality, \catt and \compresskv have the potential 
to reduce energy costs, 
to extend model access to resource-constrained settings, and 
to facilitate scientific discovery. However, we caution that any approximate attention tool should be deployed responsibly and only after evaluating the suitability and safety of the associated model. 

\section*{Acknowledgments}
The authors thank Insu Han for sharing his code, model settings, and valuable advice concerning the image generation and image classification benchmarks. TS was supported by an EPSRC-DTP scholarship, partially funded by the Department of Mathematics at Imperial College London. TS thanks G-Research for financial support to attend the conference. Part of this research was conducted during TS’s internship at Microsoft Research New England.

\bibliography{refs}
\bibliographystyle{Styles/icml2026}

\newpage
\appendix
\onecolumn
     \etoctocstyle{1}{Appendix Contents}
    \etocdepthtag.toc{mtappendix}
    \etocsettagdepth{mtchapter}{none}
    \etocsettagdepth{mtappendix}{section}
    \etocsettagdepth{mtappendix}{subsection}
    \etocsettagdepth{mtappendix}{subsubsection}
    {\small\tableofcontents}
\numberwithin{theorem}{section}
\numberwithin{lemma}{section}
\numberwithin{definition}{section}
\numberwithin{proposition}{section}
\numberwithin{corollary}{section}
\numberwithin{figure}{section}
\numberwithin{table}{section}
\numberwithin{algocf}{section}
\section{\pcref{prop:transformer-bound-via-A-approximation}}\label{subsection:proof:cat-error-decomposition}
 Introduce the shorthand $\Amin\defeq \min_{i\in[m],l\in[n]}\A_{il}$. 
 Since $\Dinv\A$ is a row-stochastic matrix, and $\CAtt = \mathrm{clip}(\hatDinv\hatA \V, \values_\mathrm{min}, \values_\mathrm{max})$ we immediately have the upper bound \begin{align}
    \smaxnorm{\O-\CAtt}\leq \smaxnorm{\vmax-\vmin}\leq 2\Vnorm.
\end{align}
We now consider two cases. First suppose that $\rownorm{\A-\hatA} \geq \sqrt{n}\Amin$.  Then $\frac{\frac{3}{\sqrt{n}}\rownorm{\A-\hatA}}{\Amin} \geq 3 > 2$, so  the advertised bound \cref{eq:clipped-catt-bound-proposition} holds.

Next suppose that $\rownorm{\A-\hatA} < \sqrt{n}\Amin$. 
In this case, $\hatD$ has all positive entries as, for each $i\in[m]$,
\begin{align}
\hatD_{ii}
    =
\D_{ii}
    +
\hatD_{ii} - \D_{ii}
    =
\e_i^\top\A\boldone_n
    +
\e_i^\top(\hatA - \A)\boldone_n
    \geq
n\min_{l\in[n]}\A_{il}
    -
\rownorm{\hatA - \A}\twonorm{\boldone_n}
    > 
0, 
\end{align}
where we have used  \Holder's inequality and the definition of $\rownorm{\cdot}$.

Now consider the decomposition
\begin{align}\label{eq:output-error-decomp}
\O - \hatO
    &=
(\Dinv\A \V - \Dinv\hatD \hatO)
    +
(\Dinv\hatD \hatO -\hatO)
    =
\Dinv(\A \V - \hatD \hatO) 
    +
(\Dinv -\hatDinv) \hatD \hatO.
\end{align}
Using \Holder's inequality, the definition of $\rownorm{\cdot}$, and the clipping of $\hatO$, we find that, for each $i\in[m]$ and $j\in[n]$, 
\begin{align}\label{eq:Dinv-difference}
|((\Dinv -\hatDinv) \hatD \hatO)_{ij}|
    =
|(\Dinv(\hatD -\D)\hatO)_{ij}|
    =
\frac{|\e_i^\top(\hatA-\A)\boldone_n|}{\e_i^\top\A\boldone_n}|\hatO_{ij}|
    \leq
\frac{\frac{1}{\sqrt{n}}\rownorm{\A-\hatA}\Vnorm}{\Amin}.
\end{align}

Next, since $\hatD$ has positive entries, we can write 
\begin{align}
\hatD \hatO = \clip(\hatA\V, \hatD\vmin, \hatD\vmax),
\quad 
\A \V = \clip(\A\V, \D\vmin, \D\vmax),
\end{align}
and therefore, by the triangle inequality, \Holder's inequality, and the definition of $\rownorm{\cdot}$
\begin{align}\label{eq:av-hatdo-diff}
\smaxnorm{\A \V - \hatD \hatO}
    &\leq
\smaxnorm{\clip(\A\V, \D\vmin, \D\vmax)
    -
\clip(\A\V, \hatD\vmin, \hatD\vmax)} \\
    &+ 
\smaxnorm{\clip(\A\V, \hatD\vmin, \hatD\vmax)
    -
\clip(\hatA\V, \hatD\vmin, \hatD\vmax)} \\
    &\leq
\smaxnorm{\D-\hatD}\Vnorm
    +
\max_{i\in[m],j\in[n]}|\e_i^\top(\A-\hatA)\V\e_j|\\
    &\leq
\sqrt{n}\rownorm{\A-\hatA}\Vnorm
    +
\max_{j\in[n]}\rownorm{\A-\hatA}\twonorm{\V\e_j}
    \leq
2\sqrt{n}\rownorm{\A-\hatA}\Vnorm.
\end{align}
Finally, since $\smaxnorm{\Dinv} \leq 1/(n\Amin)$, we conclude that $\error\leq \frac{\frac{3}{\sqrt{n}}\rownorm{\A-\hatA}\Vnorm}{\Amin}$ from \cref{eq:output-error-decomp,,eq:Dinv-difference,,eq:av-hatdo-diff}.

\section{Background on Reproducing Kernel Hilbert Spaces}\label{sec:kernel-preliminaries}
A function $h : \mathbb R^d\times \mathbb R^d \to \mathbb R$ is called a kernel if there exists a Hilbert space $(\mathcal H, \langle \cdot, \cdot\rangle_{\mathcal H})$ and a (feature) map $\Phi: \mathbb R^d \to \mathcal H$ such that $h(\x, \y) = \langle \Phi(\x), \Phi(\y)\rangle_{\mathcal H}$ \citep[Definition 4.1]{steinwart2008support}. The function $h$ is a kernel function if and only if for every dataset $\mathcal D \subseteq \mathbb R^d$ the data-matrix $\Hc \defeq h(\mathcal D, \mathcal D)$ is symmetric and positive semi-definite \citep[Theorem 4.16]{steinwart2008support}. In particular, the exponential kernel $h(\x, \y) = \exp(\beta\inner{\x}{\y})$ for $\beta >0$ is a kernel function.

We will mostly deal with finite-dimensional subspaces of $\mathcal H$. Specifically, for a subset $\mathcal C\subseteq \mathbb R^d$, the span of feature maps evaluated at the $\mathcal C$ defines a finite-dimensional sub-space of $\mathcal H$:
\begin{equation}
    \mathcal H_{\mathcal C} \defeq \overline{\set{\Phi(\x_l); \x_l \in \mathcal C}} \subseteq \mathcal H
\end{equation}
Hence, we can define a projection operator
\begin{equation}
    P_{\mathcal C} : \mathcal H \to \mathcal H_{\mathcal C}, \quad \Phi(\y)\mapsto \Phi(\mathcal C) \langle \Phi(\mathcal C), \Phi(\mathcal C)\rangle_{\mathcal H}^+ \langle \Phi(\mathcal C), \Phi(\y)\rangle_{\mathcal H}\,,
\end{equation}
where $\langle \Phi(\mathcal C), \Phi(\y)\rangle_{\mathcal H} = h(\mathcal C, \y)$, $\Hc^+$ denotes the pseudo-inverse of $\Hc$ and $\Phi(\mathcal C) = \sum_{\x_l \in \mathcal C}\Phi(\x_l)$. We identify $\langle \Phi(\mathcal C), \Phi(\mathcal C)\rangle_{\mathcal H} = h(\mathcal C, \mathcal C)$. Consequently, the inner product on this sub-space reads
\begin{equation}\label{eq:nystrom-definition}
    \langle P_{\mathcal C}\Phi(\x), P_{\mathcal C}\Phi(\y)\rangle_{\mathcal H_{\mathcal C}} = h(\x, \mathcal C)h(\mathcal C, \mathcal C)^+h(\mathcal C, \y) \defeq \hnys(\x, \y)\,.
\end{equation}
This is the so-called \emph{Nystr\"om projection} of the kernel $h$ which we will heavily employ in the following. The immediate consequence is that $\hres \defeq h-\hnys$ is also a kernel function associated with the orthogonal complement of $\mathcal H_{\mathcal C}$ in $\mathcal H$. In particular, the Cauchy-Schwarz inequality holds in both sub-spaces:
\begin{equation}
    \hnys(\x, \y) \leq \sqrt{\hnys(\x, \x)}\sqrt{\hnys(\y, \y)}\,,\quad \hres(\x, \y) \leq \sqrt{\hres(\x, \x)}\sqrt{\hres(\y, \y)}\,.
\end{equation}
\section{\pcref{prop:nystrom}}\label{section:proof:nystrom}
Let $[\Q; \K]$ denote the row-wise concatenation of $\Q$ and $\K$.
Then, using the variational formulation of $\rownorm{\cdot}$, the positive-definiteness of the residual kernel matrix $\hres([\Q;\K],[\Q;\K])$, and Cauchy-Schwarz, we find that
\begin{align}
\rownorm{\A-\hatA}
    &=
\rownorm{\hres(\Q,\K)}
    =
\max_{i\in[m]}\sup_{\u:\twonorm{\u}=1}\e_i^\top\hres(\Q,\K)\u \\
    &=
\max_{i\in[m]}\sup_{\u:\twonorm{\u}=1}[\e_i; \boldzero_n]^\top\hres([\Q;\K],[\Q;\K])[\boldzero_m;\u] \\
    &\leq
\max_{i\in[m]}\sup_{\u:\twonorm{\u}=1}\sqrt{[\e_i; \boldzero_n]^\top\hres([\Q;\K],[\Q;\K])[\e_i; \boldzero_n]}\sqrt{[\boldzero_m;\u]^\top\hres([\Q;\K],[\Q;\K])[\boldzero_m;\u] } \\
    &=
\max_{i\in[m]}\sup_{\u:\twonorm{\u}=1}\sqrt{\e_i^\top\hres(\Q,\Q)\e_i}\sqrt{\u^\top\hres(\K,\K)\u } 
    =
\sqrt{\exp(\beta\rownorm{\Q}^2)\cdot\opnorm{\hres(\K,\K)}}.
\end{align}

\section{\pcref{thm:rp-cholesky}}\label{subsection:proof:rp-cholesky}
The $\Hc$ estimate $\widehat \Hc^{r} \defeq h(\K, \coreset)\W$ produced by \rpnys is identical to that produced by the randomly pivoted Cholesky (RPC) algorithm studied in \citet{epperly2023kernel}. %
Hence, Thm.~7 of \citet{epperly2023kernel} already states a slightly looser upper bound on $\mathbb E\sspecnorm{\Hc - \widehat \Hc^{r}}$. We revisit the arguments of \citet[Thm.~7]{epperly2023kernel} to derive the refined bound in \cref{thm:rp-cholesky}.

We begin by computing the expected value of the residual kernel matrix after one iteration of \rpnys. 
Suppose $s$ is sampled according to the $\p^r$ pivoting distribution \cref{eq:rpc-pivoting-rule}, and let $\mathcal S' = \mathcal S \cup \{s\}$. Then improvement of $\widehat{\Hc}^{(r+1)}$ over $\widehat{\Hc}^{r}$ is given by
\begin{align}
    \widehat{\Hc}^{(r+1)} - \widehat{\Hc}^{r} &= h(\K, \K_{\mathcal S'})\g\, \g^\top h(\K_{\mathcal S'}, \K)\\
    &= \frac{(\hnys^{r}(\K, \keys_s) - h(\K, \keys_s))(\hnys^{r}(\K, \keys_s) - h(\K, \keys_s))^\top}{\hres^{r}(\keys_s, \keys_s)}
    = \frac{\hres^{r}(\K, \keys_s)\hres^{r}(\keys_s, \K)^\top}{\hres^{r}(\keys_s, \keys_s)}\,.
\end{align}
Defining the residual kernel matrices $\Hres^{q}\defeq \Hc-\Hhat[q]$ for $q\in\{r,r+1\}$, we therefore have
\begin{equation}
    \Earg{\Hres^{(r+1)}} = \Earg{\Hc - \Hhat[r] + \Hhat[r] - \Hhat[(r+1)]} = \Earg{\Hres^{r} - \frac{\hres^{r}(\K, \mathbf x_s)\hres^{r}(\mathbf x_s, \K)}{\hres^r(\mathbf x_s, \mathbf x_s)}} = \Hres^{r} - \frac{{\Hres^{r}}^2}{\mathrm{tr}(\Hres^{r})}\,.
\end{equation}
Using this identity one obtains the following lemma:
\begin{lemma}[Iterated expected residual bound]\label{lem:rp-cholesky-contraction-bound}
    Consider the map $\Phi(\A) = \A - \frac{\A^2}{\mathrm{tr}(\A)}$ defined for  symmetric positive definite matrices $\A$. It holds that
    \begin{equation} 
    \mathbb E \sspecnorm{\Hc - \widehat \Hc^{r}} \leq \lambda_1(\Phi^r(\Hc))\,.
    \end{equation}
\end{lemma}
\begin{proof}
    Using the tower-property of conditional expectations we have
    \begin{equation}
    \begin{aligned}
        \Earg{\Hc - \widehat \Hc^{r}} = \Earg{\Earg{\Hc - \widehat \Hc^{r}\mid  \Hc^{(r-1)}} } 
        = \Earg{\Phi(\Hc - \widehat \Hc^{(r-1)})}
        \preceq \Phi \paren{\Earg{\Hc - \widehat \Hc^{(r-1)}}}
    \end{aligned}
    \end{equation}
    where in the final step we used Jensen's inequality and the concavity of $\Phi$  \citep[Lem.~5.3]{chen2022randomly}. By the same lemma, $\Phi$ is monotone in the sense that $\A \preceq \B \implies \Phi(\A) \preceq \Phi(\B)$, and we can thus iterate the argument and conclude $\Earg{\Hc - \widehat \Hc^{r}} \preceq \Phi^{r}(\Hc)$. We now have for any $\keys\in \reals^n$ with $\norm{\keys}_2\leq 1$
    \begin{equation}
        \Earg{ \keys^\top (\Hc - \widehat \Hc^{r})\keys} = \keys^\top\Earg{\Hc - \widehat \Hc^{r}} \keys  \leq \keys^\top\Phi^{r}(\Hc)\keys \leq \lambda_1(\Phi^{r}(\Hc))\,.
    \end{equation}
\end{proof}

Our proof will also make use of a second lemma that bounds the maximum eigenvalue of $\Phi^r(\Hc)$ in terms of an ordinary differential equation.
\begin{lemma}[Differential equation bound]\label{lem:rp-cholesky-ode-bound}
    Let $q\in [n]$ be arbitrary. Then, $\lambda_1(\Phi^r(\Hc)) \leq \eta(r)$, where $\eta$ is the decreasing solution of the ordinary differential equation
    \begin{equation}\label{eq:ode-proof-rpcholesky}
        \frac{\mathrm d\eta(t)}{\mathrm dt} = -\frac{\eta(t)^2}{q \eta(t) + \sum_{l = q + 1}^n \lambda_l(\Hc)} \quad \text{ with }\quad \eta(0) = \lambda_1(\Hc)\,.
    \end{equation}
\end{lemma}
\begin{proof}
    We follow the proof of \citet[Thm.~7]{epperly2023kernel}. Firstly, the ordered eigenvalues of $\Phi^r(\Hc)$ are non-negative and satisfy the following recurrence relation:
    \begin{equation}\label{eq:eigenvalue-recurrence-contraction}
        \lambda_i(\Phi^r(\Hc)) = \lambda_i(\Phi^{r-1}(\Hc)) - \frac{{\lambda_i(\Phi^{r-1}(\Hc))}^2}{\sum_{l=1}^n \lambda_l(\Phi^{r-1}(\Hc))}\,.
    \end{equation}
    It follows that $\lambda_i(\Phi^r(\Hc))\leq \lambda_i(\Phi^{r-1}(\Hc))$. In addition, one can show from this that $\lambda_{i + 1}(\Phi^r(\Hc)) \leq \lambda_{i}(\Phi^r(\Hc))$, i.e.,  the recurrence relation preserves the ordering of the eigenvalues \citep[proof of Thm.~7]{epperly2023kernel}. Using these two facts, one can bound the trace of $\Phi^{r-1}(\Hc)$ as
    \begin{equation}
        \sum_{i=1}^n \lambda_i(\Phi^{r-1}(\Hc)) \leq q\lambda_1(\Phi^{r-1}(\Hc)) + \sum_{i = q+1}^n \lambda_i(\Hc)\,.
    \end{equation}
    Plugging this into \cref{eq:eigenvalue-recurrence-contraction} yields
    \begin{equation}
        \lambda_1(\Phi^r(\Hc)) \leq \lambda_1(\Phi^{r-1}(\Hc)) -\frac{\lambda_1(\Phi^{r-1}(\Hc))^2}{q\lambda_1(\Phi^{r-1}(\Hc)) + \sum_{i = q+1}^n \lambda_i(\Hc)}\,.
    \end{equation}
    Let now $\eta(t)$ be the solution to the ordinary differential equation in \cref{eq:ode-proof-rpcholesky}. Since $\mathrm d\eta(t)/\mathrm dt$ is negative, the solution $\eta(t)$ is monotonically decreasing. Since additionally, $x\mapsto - \frac{x^2}{qx + \sum_{i = q+1}^n \lambda_i(\Hc)}$ is decreasing it follows that $\mathrm d\eta(t)/\mathrm dt$ is decreasing and consequently
    \begin{equation}
        \eta(r+1) = \eta(r) + \int_r^{r+1} \frac{\mathrm d\eta(t)}{\mathrm dt} \mathrm dt \leq \eta(r) - \frac{\eta(r)^2}{q \eta(r) + \sum_{i = q+1}^n \lambda_i(\Hc)}\,.
    \end{equation}
    Finally, we notice that the function
    \begin{equation}
       x\mapsto \varphi(x) \defeq x - \frac{x^2}{qx + \sum_{i = q+1}^n \lambda_i(\Hc)}
    \end{equation}
    is monotonically increasing in $x$ for any $q\geq 1$. Hence, we conclude inductively from $\lambda_1(\Hc) \leq \eta(0)$, $\lambda_1(\Phi^{r+1}(\Hc)) \leq \varphi(\lambda_1(\Phi^r(\Hc)))$, $\eta(r+1) \leq \varphi(\eta(r))$, and the monotonicity of $\varphi$ that $\lambda_1(\Phi^{r}(\Hc))\leq \eta(r)$ for all $r\in \mathbb N$.
\end{proof}
Finally, we will use a variant of the Eckart-Young-Mirsky \citep{mirsky} theorem for the trace norm:
\begin{lemma}[Eckart-Young-Mirsky for nuclear norm]\label{lem:EckartYoungMirsky}
    Let $\A$ be positive semi-definite and let $\lambda_1(\A) \geq \lambda_2(\A) \geq \dots\geq \lambda_n(\A)\geq 0$ be the ordinally sorted eigenvalues of $\A$. Then,
    \begin{equation}
        \sum_{i = q+1}^n \lambda_i(\A) = \min_{{\begin{smallmatrix} \mathrm{rank}(\Phi)\leq q\\
        \Phi\Phi^\top \preceq \A
        \end{smallmatrix}}} \mathrm{tr}(\A - \Phi\Phi^\top)
    \end{equation}
\end{lemma}
\begin{proof}
    Let $\proj_q(\A)$ be the best rank $q$ approximation of $\A$. By the Eckart-Young-Mirsky theorem for the spectral norm it holds for $i = 1, 2, \dots, n$ that
    \begin{equation}
        \lambda_1(\A - \proj_q(\A)) = \lambda_{q+1}(\A)\,
    \end{equation}
    and consequently for any $\Phi\Phi^\top \preceq \A$
    \begin{equation}
       \sum_{i=1}^n\lambda_i(\A - \Phi\Phi^\top) = \sum_{i=1}^n\lambda_1(\A - \Phi\Phi^\top - \proj_{i-1}(\A - \Phi\Phi^\top))\,.
    \end{equation}
    Next, we use that the matrix $\Phi\Phi^\top + \proj_{i-1}(\A - \Phi\Phi^\top)$ is at most a matrix of rank $q + i -1$. Thus, we get
    \begin{equation}
        \lambda_1(\A - \Phi\Phi^\top - \proj_{i-1}(\A - \Phi\Phi^\top)) \geq \lambda_1(\A - \proj_{q+i -1}(\A)) = \lambda_{q+i}(\A)
    \end{equation}
    Consequently, it follows
    \begin{equation}
        \sum_{i=1}^n \lambda_i(\A - \Phi\Phi^\top) \geq \sum_{i=1}^n \lambda_{q+i}(\A) = \sum_{i = q +1}^n \lambda_i(\A)\,.
    \end{equation} 
\end{proof}
We now conclude with a proof of the theorem statement:
\paragraph{\pcref{thm:rp-cholesky}}
\begin{proof}
From \cref{lem:rp-cholesky-contraction-bound} and \cref{lem:rp-cholesky-ode-bound} we have after $r$ iterations of \rpnys
\begin{equation}
    \mathbb E\sspecnorm{\Hc - \widehat \Hc^{r}} \leq \lambda_1(\Phi^{r}(\Hc)) \leq \eta(r)\,.
\end{equation}
Next, we find a time parameter $r_\eps$ such that $\eta(r_\eps) = \varepsilon$. By a separation of variables we have for $q\in [n]$
\begin{equation}\label{eq:separation-ode}
\begin{aligned}
     r_\eps = \int_{\lambda_1(\Hc)}^{\varepsilon} \left(-\frac{q \eta + \sum_{i = q + 1}^n \lambda_i(\Hc)}{\eta^2}\right) \mathrm d\eta 
     &= q \log\left(\frac{\lambda_1(\Hc)}{ \varepsilon}\right) + 
\left(\sum_{i = q + 1}^n \lambda_i(\Hc)\right)\left(\frac{1}{\varepsilon}- \frac{1}{\lambda_1(\Hc)}\right)
\end{aligned}
\end{equation}
From \cref{lem:EckartYoungMirsky} we have that $\sum_{i = q + 1}^n \lambda_i(\Hc) \leq \nuc{\Hc - \mathbf T}$ for any $\mathbf T\preceq \Hc$ with $\mathrm{rank}(\mathbf T) = q$.
Since $\eta$ is decreasing in its argument, and $r \geq r_\eps$, we have  $\eta(r)\leq \eta(r_\eps) = \varepsilon$ yielding the claim.
\end{proof}

\section{\pcref{cor:trace-polynomial-guarantees}}\label{subsection:taylor-bound}
The guarantees of \rpnys depend on the low-rank approximability of $\Hc$ through $\mathrm{rank}(\T)$ and $\nuc{\Hc - \T}$ for $\T \preceq \Hc$. We determine explicit expressions for these to quantities via a Taylor approximation of the exponential kernel. For a dataset $\K = (\keys_l)_{l\in [n]}$, define the kernel matrix $\Hc \defeq h(\K, \K)$, where $h(\keys_1, \keys_2) \defeq \exp(\beta \inner{\keys_1}{\keys_2})$ is the exponential kernel. 
To construct a suitable low-rank approximation, we find the polynomial expansion of the exponential kernel following \citet{cotter2011explicit}. We use the notation of multi-indices $\alpha \in \mathbb N_0^d$:
\begin{equation}
    \begin{aligned}
        \vert \alpha \vert = \alpha_1 + \alpha_2 + \dots + \alpha_d \qtext{and}
        \begin{pmatrix}
            s \\
            \alpha
        \end{pmatrix} = \frac{s!}{\alpha_1! \alpha_2! \dots \alpha_d!}\,.
    \end{aligned}
\end{equation}
\begin{lemma}[Exponential kernel feature expansion]\label{prop:features_attentionkernel}
    The exponential kernel has the following expansion into features of rank one:
    \begin{equation}
        \exp(\beta \inner{\keys_1}{ \keys_2}) = \sum_{s = 0}^\infty \sum_{\vert \alpha \vert = s} \phi_{\alpha}(\mathbf k_1)\phi_{\alpha}(\mathbf k_2)
    \end{equation}
    with
    \begin{equation}
        \boldsymbol{\phi}_{\alpha}(\mathbf k) = \sqrt{\frac{1}{\lvert \alpha\rvert !}\binom{\lvert \alpha\rvert}{\alpha}\beta^{\lvert \alpha \rvert }}\mathbf k^\alpha
    \end{equation}
\end{lemma}
\begin{proof}
A Taylor expansion of the exponential shows
    \begin{equation}
\begin{aligned}
    \hatt(\mathbf k, \mathbf k') &= \sum_{s=0}^\infty \frac{1}{s!}\left(\sum_{j = 1}^d \beta\mathbf k_j\mathbf k'_j\right)^s  
    = \sum_{s=0}^\infty \frac{1}{s!}\left(\sum_{\vert \alpha\vert = s} \binom{s}{\alpha} \beta^s\mathbf k^\alpha \mathbf {k'}^{\alpha}\right) 
    = \sum_{s=0}^\infty \sum_{\vert \alpha\vert = s}\frac{1}{s!} \binom{s}{\alpha} \beta^s\mathbf k^\alpha \mathbf {k'}^{\alpha}\,.
\end{aligned}
\end{equation}
\end{proof}
We then have the following low-rank approximation result for $\Hc = h(\K, \K)$: 
\begin{lemma}[Error from Taylor polynomial truncation]\label{prop:polynomial-approximation} 
    Let $\T^s$ be the order $s$ Taylor approximation of the exponential kernel with entries $\T^s_{il} \defeq \sum_{\vert \alpha\vert \leq s} \phi_{\alpha}(\keys_i)\phi_{\alpha}(\keys_l)$. Then, $\T^s$ has rank $\leq\binom{s+d}{d}$ and satisfies $\T^s\preceq \Hc$. Furthermore, the Taylor residual satisfies
    \begin{equation}\label{eq:taylor-residual}
        \nuc{\Hc - \T^s} \leq n\exp(\beta \rownorm{\K}^2)\paren{\frac{e\beta\rownorm{\K}^2}{s+1}}^{s+1}\,.
    \end{equation}
\end{lemma}
\begin{proof}
By definition, $\T_{il}^s$ is the sum of at most $\# \set{\alpha \in \mathbb N_0^d ; \lvert \alpha\rvert \leq s} = \binom{s+d}{d}$ rank one matrices. 
We bound the trace by the worst case approximation error of any of the entries of $\Hc$:
\begin{equation}
    \mathrm{tr}(\Hc - \T^s) \leq n \max_{i, l\in [n]}\left\lvert h(\keys_i, \keys_l)  - \T^s_{il} \right\rvert \,.
\end{equation}
Without loss of generality let this maximum be attained for $\keys_i, \keys_l$ where $i, l\in [n]$. By Taylor's theorem, the error between the exponential kernel and it's rank $s$ approximation is bounded by
\begin{align}
    \left\lvert \exp(\beta \inner{\keys_i}{\keys_l})  - \sum_{\vert \alpha \vert \leq s} \phi_{\alpha}(\mathbf k_1)\phi_{\alpha}(\mathbf k_2) \right\rvert 
    &=\sum_{i = s+1}^\infty \frac{1}{i!}\left(\beta \inner{\keys_i}{ \keys_l}\right)^i 
    \leq \frac{\exp(\beta \inner{\keys_i}{ \keys_l})}{(s+1)!}\left(\beta\rownorm{\K}^2\right)^{s+1}\,\\
    &\leq \exp(\beta \inner{\keys_i}{ \keys_l})\left(\frac{e\beta\rownorm{\K}^2}{s+1}\right)^{s+1}
\end{align}
where we first invoked the bound of the Taylor residual of the exponential function on $\reals$, and afterwards used the lower bound of the factorial $n!> (n/e)^n$. Hence, we obtain for the polynomial approximation up to order $s$:
\begin{equation}
\begin{aligned}
    \lvert h(\keys_i, \keys_l) - \T^s_{il}\rvert \leq \exp\left(\beta \inner{\keys_i}{ \keys_l}\right)\left(\frac{e\beta\rownorm{\K}^2}{s+1}\right)^{s+1} \leq \exp(\beta \rownorm{\K}^2)\left(\frac{e\beta\rownorm{\K}^2}{s+1}\right)^{s+1}\,.
\end{aligned}
\end{equation}
\end{proof}

We are now in the position to invert the bound on the trace-norm, which is the main result of this section:
\paragraph{\pcref{cor:trace-polynomial-guarantees}}
\begin{proof}
First, assume $\tau = 1$ such that $\Hc_\tau = \Hc$. Define $z \defeq \log(n \exp(\beta \RK^2)/\varepsilon)$ and write 
\begin{equation}
    {\tilde s} \defeq e \beta\RK^2 \exp\left(\LambW*{\frac{z}{e\beta\RK^2}}\right)\,,
\end{equation}
    We use the identities $(e\beta\RK^2/{\tilde s})^{\tilde s} = \exp({\tilde s} \log(e\beta\RK^2/{\tilde s}))$ and $z = \exp(\LambW{z})\LambW{z}$ to compute
    \begin{equation}
    \begin{aligned}
    \exp\left({\tilde s}\log\left(\frac{e\beta\RK^2}{{\tilde s}}\right)\right) & = \exp\left(-e\beta\RK^2 \exp\left(\LambW*{\frac{z}{e\beta\RK^2}}\right)\LambW*{\frac{z}{e\beta\RK^2}}\right)
    = \exp\left(- z\right) 
    = \frac{\varepsilon}{n \exp(\beta \RK^2)}\,.
    \end{aligned}
    \end{equation}
    Next, let $s \defeq \lfloor {\tilde s} \rfloor$.
    Note that $t \mapsto (e\beta\RK^2/t)^t$ is maximised at $t = \beta\RK^2$ and decreasing for $t > \beta\RK^2$. Furthermore, ${\tilde s}\geq \beta\RK^2$ by definition. Hence,
    \begin{equation}
       n\exp(\beta\RK^2)  \left(\frac{e\beta\RK^2}{s+1}\right)^{s+1} \leq n\exp(\beta\RK^2) \left(\frac{e\beta\RK^2}{{\tilde s}}\right)^{{\tilde s}} = n\exp(\beta\RK^2)\exp(-z) = \varepsilon\,.
    \end{equation}
    Thus, the claim follows from \cref{prop:polynomial-approximation} and a rescaling by $\tau$ $\keys \to \keys/\tau$, $\Hc \to \Hc_\tau$, $\RK \to \RK/\tau$. 
\end{proof}
\section{\pcref{lem:binom}}\label{proof-lem:binom}
Since $\rank{\T^s} \leq \binom{s+d}{d}$ by \cref{prop:polynomial-approximation}, the result follows immediately from the following more precise bound on the binomial coefficient:
\begin{lemma}[Bounds for the binomial coefficient]\label{lem:binom}
    For any $s,d\in\naturals$, it holds
    \begin{equation}
        \binom{s+d}{d} \leq \frac{1}{\sqrt{2 \pi}} \sqrt{\frac{1}{s} + \frac{1}{d}} \exp\paren{(s+d)\mathrm{Ent}\paren{\frac{s}{s+d}}}
    \end{equation}
    with $\mathrm{Ent}(p) = -p\log(p) - (1-p)\log(1-p)$.
\end{lemma}
\begin{proof}
We have by Robbin's version of Stirlings formula \citep{robbins}:
\begin{equation}
\sqrt{2\pi n}\left(\frac{n}{e}\right)^n e^{\frac{1}{12n + 1}} \leq n! \leq \sqrt{2\pi n}\left(\frac{n}{e}\right)^n e^{\frac{1}{12n}}\,\qtext{for all} n\in\naturals.
\end{equation}
Thus, invoking these bounds shows in a direct calculation
\begin{equation}
\begin{aligned}
    \binom{s+d}{d} 
    =
    \frac{(s+d)!}{s!\,d!}
    &\leq \frac{\sqrt{2\pi}\sqrt{s+d} (s+d)^s\, (s+d)^d}{2\pi \sqrt{s d} s^s\, d^d}= \frac{1}{\sqrt{2 \pi}} \sqrt{\frac{1}{d} + \frac{1}{s}} \left(1 + \frac{s}{d}\right)^d \left(1 + \frac{d}{s}\right)^s\,.
\end{aligned}
\end{equation}
Finally, we identify that
\begin{equation}
    \left(1 + \frac{s}{d}\right)^d \left(1 + \frac{d}{s}\right)^s = \exp\paren{(s+d)\mathrm{Ent}\paren{\frac{s}{s+d}}}\,.
\end{equation}
\end{proof}

\section{Choosing a Rescaling Parameter for Keys and Queries}\label{subsection:rescaling-rule}
In this section we describe a rescaling rule of the form $\queries \mapsto \queries \tau$, $\keys \mapsto \keys/\tau$, which we use in our empirical and theoretical results. The key idea is the following: Suppose we perform the rescaling prescribed above and define the approximate attention matrix as the Nystr\"om approximation obtained from rescaled queries and keys:
\begin{equation}\label{eq:hatAtau}
    \hatA_\tau \defeq h(\tau\queries, \tau^{-1}\coreset)h(\tau^{-1}\coreset, \tau^{-1}\coreset)^{-1}h(\tau^{-1}\coreset, \tau^{-1}\K)\,.
\end{equation}
By \cref{prop:nystrom}, we find for \rpnys applied to $\Hc_\tau \defeq h(\tau^{-1}\K, \tau^{-1}\K)$:
\begin{equation}
    \rownorm{\A - \hatA_\tau} \leq \exp\paren{\frac{\tau^2\beta}{2} \RQ^2} \sqrt{\sspecnorm{\Hc_\tau - \widehat{\Hc}_\tau}}\,.
\end{equation}
Increasing $\tau$ makes the matrix $\Hc_\tau$ increasingly low-rank approximable. Note that in the extreme case, $\Hc_\tau \xrightarrow{\tau\to\infty} \boldone_n \boldone_n^\top$ becomes a rank one matrix. On the other hand, increasing $\tau$ comes at the cost of the error inflation factor $\exp\big(\frac{\tau^2\beta}{2} \RQ^2\big)$. 
 We therefore aspire to strike a balance between the low-rank approximability of $\Hc_\tau$ and the control of  $\exp\big(\frac{\tau^2\beta}{2} \RQ^2\big)$. We start the analysis by characterising the rowwise approximation error of $\A$ in terms of the ratio $\RQ /\RK$ and $\tau$. We combine \cref{prop:nystrom}, \cref{thm:rp-cholesky}, and \cref{cor:trace-polynomial-guarantees} to obtain the following rowwise approximation guarantee for $\A_l = \exp(\beta \queries^\top \keys_l)$:
\begin{lemma}[Rescaling guarantee]\label{lem:basis-change}
    Let $\coreset \subseteq \K$ and $\W$ be the coreset and Nystr\"om weights outputted by $\rpnys$ for the kernel function $h_\tau(\keys_1, \keys_2)= \exp\paren{\frac{\beta}{\tau^2}\inner{\keys_1}{\keys_2}}$. Define the associated rank-$r$ Nystr\"om approximation $\hatA^r_\tau \defeq h(\queries, \coreset)\W$ of $\A$. Then,
    \begin{equation}\label{eq:runtime-parameter-lower-bound}
        \mathbb E\rownorm{\A - \hatA_\tau^{r}} \leq \varepsilon \quad \text{for any} \quad r \geq \binom{s(\varepsilon)+d}{d}(\log(\specnorm{\Hc}\exp(\beta\tau^2\RQ^2)/\varepsilon^2) + 1
    \end{equation}
    provided
    \begin{equation}\label{eq:equivalent-s-lower-bound}
        s(\varepsilon) 
        \geq 
        \floor*{e\beta \RQ\RK \frac{\frac{1}{e}(\rho^2 + b(\varepsilon)\rho + 1)}{\rho \LambW*{\frac{1}{e}(\rho^2 + b(\varepsilon)\rho + 1)}}}\,,
    \end{equation}
    where $b(\varepsilon) = \frac{\log(n/\varepsilon^2)}{\beta\RQ\RK}$ and $\rho = \tau^2\RQ/\RK$.
\end{lemma}
\begin{proof}
    We first apply \cref{prop:nystrom} to rescaled datasets $\tau\Q$ and $\tau^{-1}\K$. 
    From \cref{prop:nystrom} it immediately follows from a rearrangement of the bound %
    that
    \begin{equation}
    \mathbb E\specnorm{\hres(\tau^{-1}\K, \tau^{-1}\K)} \leq \varepsilon^2\exp(-\beta\tau^2\RQ^2)\,
    \quad\Rightarrow\quad
        \mathbb E\rownorm{\A - \hatA_\tau^{r}} \leq \varepsilon.
    \end{equation}
    Using $\varepsilon^2\exp(-\beta\tau^2\RQ^2)$ in place of $\varepsilon$ in \cref{thm:rp-cholesky} and \cref{prop:polynomial-approximation} we can get the claimed requirement for the runtime parameter in \cref{eq:runtime-parameter-lower-bound} for a polynomial with large enough rank $q = \binom{s+d}{d}$ such that
    \begin{equation}
        \nuc{\Hc_\tau - \T^s} \leq \exp(-\beta\tau^2\RQ^2)\varepsilon^2\,.
    \end{equation}
    \cref{cor:trace-polynomial-guarantees} states that a sufficient condition for this trace-bound is
    \begin{equation}\label{eq:s-lower-bound}
        s \geq \floor*{e \beta\frac{\RK^2}{\tau^2} \exp\left(\LambW*{\frac{\log(n \exp(\beta\tau^2\RQ^2 +\beta\tau^{-2}\RK^2)/\varepsilon^2)}{e\beta\tau^{-2}\RK^2}}\right)}\,.
    \end{equation}
    At this point we can express the argument of the product logarithm in terms of the relative scales between $\RQ$ and $\RK$
    \begin{equation}
        \tau^{-2}\RK^2 = \tau^{-2}\frac{\RK}{\RQ}\RQ\RK = \frac{\RQ \RK}{\rho}\, \quad \text{and}\quad \tau^{2}\RQ^2 = \tau^{2}\frac{\RQ}{\RK}\RQ\RK = \rho\RQ\RK
    \end{equation}
    which yields
    \begin{equation}
        \frac{\log(n \exp(\beta\tau^2\RQ^2 +\beta\tau^{-2}\RK^2)/\varepsilon^2)}{e\beta\tau^{-2}\RK^2} = \frac{\log(n/\varepsilon^2)\rho + \beta\RQ\RK \rho^2 + \beta\RQ\RK}{e\beta\RQ\RK}\,.
    \end{equation}
    Invoking the identity $\exp(\LambW{z}) = z/\LambW{z}$ yields for any $\rho$ the  equivalent expression \cref{eq:equivalent-s-lower-bound} for \cref{eq:s-lower-bound} as claimed.
\end{proof}
\begin{remark}\label{remark:connection-tau-rho}
    The result in \cref{lem:basis-change} describes guarantees in terms of $\rho = \tau^2 R_\Q/R_\K$. The temperature $\tau$ is a \emph{free} parameter in our algorithm that we can tune to obtain improved performance or guarantees.
    Let $\RQ, \RK$ be fixed, define $\rho_{\mathrm{in}} \defeq \RQ/\RK$, and let $\rho_{\mathrm{out}} >0$. 
    Then we have \cref{lem:basis-change} with $\rho = \rho_{\mathrm{out}}$ instead of $\rho = \rho_{\mathrm{in}}$ by running \rpnys on $\exp(\beta \K\K^\top/\tau^{2})$ with $\tau^2 = \rho_{\mathrm{out}}/\rho_{\mathrm{in}}$. The value of $\rho$ is fixed in the theoretical analysis once we invoke the Cauchy-Schwartz inequality in \cref{prop:nystrom} and isolate the data matrix $\Hc_\tau$. To summarise, the approximation $\hatA_\tau$ and the guarantees in \cref{lem:basis-change} have the following connection:
\begin{equation}
   \text{find } \rho_{\mathrm{out}}\, \to\,  \text{ define }\tau \defeq \sqrt{\rho_{\mathrm{out}} \frac{\RK}{\RQ}} \, \to \,  \text{ run \rpnys with } h_\tau = \exp\paren{\frac{\beta}{\tau^2}\<\cdot, \cdot \>} \, \xrightarrow{\cref{eq:hatAtau}}\, \hatA^{(r)}_\tau \to \text{\cref{lem:basis-change} with } \rho=\rho_{\mathrm{out}}\,.
\end{equation}
\end{remark}

Finding the right value for $\tau$ turns out to be empirically important: with no temperature adjustment and insufficient choice of $r$, the Hilbert space $\mathcal H_{\K}$ associated with the data-kernel features is poorly approximated by $\mathcal H_{\mathcal S}$. This leads to large outliers in the Nystr\"om weights $h(\coreset, \coreset)^{-1}h(\coreset, \K)$ and poor recovery of the target attention matrix. When $\rho$ is too large, on the other hand, the \rpnys selection algorithm converges to uniform sampling without replacement, a less accurate choice. We now describe tight estimates for $\rho$ that minimise the dominating rank parameter $\binom{s+d}{d}$ in \cref{lem:basis-change}.
\begin{lemma}[Optimal choice of $\rho$]\label{lem:ode-rescaling-keys-queries}
For each $b \geq 0$, 
define the optimisation objective
\begin{equation}
    l_b(\rho) \defeq \frac{\frac{1}{e}(\rho^2 + b\rho + 1)}{\rho\LambW*{\frac{1}{e}(\rho^2 + b\rho + 1)}}\,
\end{equation}
over $\rho > \sqrt{2}$. 
The unique minimiser of $l_0$ is given by
\begin{equation}\label{eq:rho0}
    \rho_0 
    \defeq 
    \sqrt{1 + e^{\LambW{2/e^2}+2}}\,.
\end{equation}
More generally, $l_b$ is uniquely minimised by a solution $\rho_b$ of the ordinary differential equation
\begin{equation}\label{app:eq:ode-rescaling-keys-queries}
\dd{b}\rho_b 
=     \frac{\rho_b^2}{(\rho_b^2
    +1)\log(\rho_b^2-1)}
= \frac{\rho_b}{2\rho_b+b}\left(1 - \frac{2}{\rho_b^2 + 1}\right) > 0 \quad \text{with} \quad \rho_0 \qtext{as in} \cref{eq:rho0}.
\end{equation}
Furthermore, the minimal value of $l_b$ satisfies
\begin{equation}
    l_b(\rho_b) = \frac{\rho_b^2 - 1}{e \rho_b} < \frac{\rho_b}{e}\,.
\end{equation}
\end{lemma}
\begin{proof}
Our proof proceeds in five steps.
First, we characterise the stationary points of $\log l_b$.
Second, we prove that any stationary point is a unique global minimizer of $l_b$.
Fourth, we derive the closed-form minimiser of $l_0$ \cref{eq:rho0}.
Fifth, we derive the differential equation \cref{app:eq:ode-rescaling-keys-queries} defining the minimisers of $l_b$.
Finally, we compute and bound the minimal value of $l_b$.

\paragraph{Characterising the stationary points}
We begin by characterising the stationary points of $\log l_b$. 
    Using the identity $\exp(\LambW{x}) = x/\LambW{x}$, we find that $\log l_b(\rho) = \LambW*{\frac{1}{e}q(b,\rho)} - \log(\rho)$ for $q(b,\rho)\defeq\frac{1}{e}(\rho^2 + b\rho + 1)$.
    The derivatives of the component functions are, using \cref{eq:derivative-lambW},
    \begin{equation}
        \dd q \LambW*{q} = \frac{\LambW{q}}{q(1 + \LambW{q})} \quad\text{and}\quad \dd\rho q(b,\rho) = \frac{1}{e}(2\rho + b)\,
    \end{equation}
    Using the identity
    $q(b,\rho) - \frac{1}{e}(2\rho^2 + b\rho) = -q(b,\rho) + \frac{1}{e}(b\rho + 2)$ we therefore obtain the equivalent stationary point conditions
    \begin{align}
        &\dd{\rho}\log l_b(\rho) = \dd{\rho}\LambW{q(b,\rho)} - \frac{1}{\rho} 
        = \frac{\frac{1}{e}(2\rho + b)\LambW*{q(b,\rho)}}{q(b,\rho) (\LambW*{q(b,\rho)} +1)} - \frac{1}{\rho} = 0 \\
        \iff\quad
        &\frac{1}{e}(2\rho^2 + b\rho) 
        = q(b,\rho) + \frac{q(b,\rho)}{\LambW{q(b,\rho)}}\,\\
        \iff\quad
        &g(b,\rho)\defeq\frac{q(b,\rho)}{\LambW{q(b,\rho)}} - q(b,\rho) + \frac{1}{e} b\rho + \frac{2}{e} = 0\,       \label{app:eq:rescaling-first-order-condition} \\
        \iff\quad
        &\LambW{q(b,\rho)} 
            = 
        \frac{e\,q(b,\rho)}{\rho^2-1}. \label{eq:LambW-at-optimality}
    \end{align}
    Furthermore, for $b \geq 0$ and $\rho > 0$ we have $q(b, \rho)>0$, and repeating the above calculations keeping track of the sign we find that
    \begin{equation} \mathrm{sgn}\paren{\dd{\rho}\log l_b(\rho)} = \mathrm{sgn}\paren{-\frac{g(b, \rho)}{q(b, \rho)}} = \mathrm{sgn} \paren{-\frac{1}{\LambW{q(b, \rho)}} + \frac{\rho^2 - 1}{eq(b, \rho)}}\,.
    \end{equation}

    \paragraph{Stationarity implies optimality}
    Now suppose that for some $b\geq 0$ there is a stationary point $\rho_b \in \set{\rho > \sqrt{2}; g(b, \rho) = 0}$. We will show that $\rho_b$ is the unique global minimiser of $l_b(\rho)$ on $(\sqrt{2}, \infty)$. For $b \geq 0$ and all $\rho > \sqrt{2}$, the map $\rho \mapsto -\frac{1}{\LambW{q(b, \rho)}} + \frac{\rho^2 - 1}{eq(b, \rho)}$ is increasing because of
    \begin{align}
        \pdd\rho q(b, \rho) &= \frac{1}{e}(2\rho + b) > 0, \\
        \pdd \rho \paren{\frac{\rho^2 -1}{eq(b, \rho)}} &= \frac{2\rho(\rho^2 + b\rho + 1) - (\rho^2 - 1)(2\rho + b)}{e^2 q(b, \rho)^2} = \frac{b\rho^2 + 4\rho + b}{e^2 q(b, \rho)^2} >0\,,
    \end{align}
    and the monotonicity of the Lambert-W function. Therefore $l(\rho)$ is increasing/decreasing for $\rho \gtrless \rho_b$ since
    \begin{equation}
\mathrm{sgn}\paren{\dd{\rho}\log l_b(\rho)} = \mathrm{sgn} \paren{-\frac{1}{\LambW{q(b, \rho)}} + \frac{\rho^2 - 1}{eq(b, \rho)}} \gtrless \mathrm{sgn} \paren{-\frac{1}{\LambW{q(b, \rho_b)}} + \frac{\rho_b^2 - 1}{eq(b, \rho_b)}} = 0\,.
    \end{equation}
    Hence, $\rho_b$ must be the only stationary point and the unique global minimiser. 

    \paragraph{Closed-form minimiser of $l_0$}
    Next, we identify the closed-form minimiser of $l_0$. 
    For $b=0$, the condition $\frac{q}{\LambW{q}} - q + \frac{2}{e} = 0$ is satisfied 
    at
    \begin{equation}
        q_0 = \frac{2}{e} + e^{\LambW*{\frac{2}{e^2}} + 1},
    \end{equation}
    and the quadratic equation $q(0,\rho_0) = q_0$ is satisfied by $\rho_0 \defeq \sqrt{e q_0 - 1} \approx 3.19$.
    
    \paragraph{Differential equation for $\rho_b$}
    The optimality condition \cref{eq:LambW-at-optimality} additionally implies that
    \begin{align}
    &q(b, \rho) 
        =
    \exp(\LambW{q(b, \rho)})\LambW{q(b, \rho)}
        =
    \frac{e q(b, \rho)}{\rho^2-1}\exp\bigg(\frac{e q(b, \rho)}{\rho^2-1}\bigg),\\
    &(\rho^2-1)\log((\rho^2-1)/e) = e q(b, \rho)
    = \rho^2+b\rho+1, \qtext{and}\\
    &b = \frac{(\rho^2-1)\log((\rho^2-1)/e) - (\rho^2+1)}{\rho}
    = \frac{(\rho^2-1)\log(\rho^2-1)}{\rho} - 2\rho.
    \end{align}
    Since $\rho\mapsto b_\rho \defeq \frac{(\rho^2-1)\log(\rho^2-1)}{\rho} - 2\rho$ is continuously differentiable  with range $(-2\sqrt{2},\infty)$ and $\dd{\rho}b_\rho = \frac{(\rho^2
    +1)\log(\rho^2-1)}{\rho^2} > 0$ for $\rho > \sqrt{2}$, the inverse function theorem \citep[Thm.~31.1]{Price1984} implies that there exists a continuously differentiable inverse function $b\mapsto\rho_b$ on $b> -2\sqrt{2}$ with 
    \begin{align}
    \dd{b}\rho_b
        = 
    \frac{1}{b'(\rho_b)}
        = 
    \frac{\rho_b^2}{(\rho_b^2
    +1)\log(\rho_b^2-1)}
        =
    \frac{\rho_b(\rho_b^2-1)}{(\rho_b^2
    +1)(b+2\rho_b)}
        =
    \frac{\rho_b}{2\rho_b+b}\paren{1-\frac{2}{\rho_b^2
    +1}}.
    \end{align}
    Furthermore, $\rho_0 > 0$, $\dd{b}\rho_b \vert_{b=0} >0$, and $\dd{b}\rho_b$ is increasing in $\rho_b$ since
    \begin{equation}
       \dd \rho \paren{\frac{\rho}{2\rho+b}\paren{1-\frac{2}{\rho^2
    +1}}} = \frac{b (\rho^4 + 4 \rho^2 - 1) + 8 \rho^3}{(\rho^2 + 1)^2 (b + 2 \rho)^2} > 0\,.
    \end{equation}
    Consequently, $\dd{b}\rho_b > 0$ for all $b\geq 0$. 
    This establishes the claim \cref{app:eq:ode-rescaling-keys-queries}. Finally, at the point of optimality we have from the first order condition \cref{app:eq:rescaling-first-order-condition}
    \begin{equation}
        l_b(\rho_b) = \frac{1}{\rho_b}\frac{q(b,\rho_b)}{\LambW{q(b,\rho_b)}} = \frac{1}{\rho_b}\paren{q(b,\rho_b) - \frac{1}{e}b\rho_b - \frac{2}{e}} = \frac{\rho_b^2 - 1}{e\rho_b}\,.
    \end{equation}
\end{proof}
A tight upper and lower bound of this differential equation can be solved in closed form:
\begin{corollary}[Bounds on optimal $\rho$]\label{cor:basis-change-bounds}
For $b > 0$, a solution $\rho_b$ of the ordinary differential equation \cref{app:eq:ode-rescaling-keys-queries} obeys the bounds
    \begin{equation}
        \max\left(\rho_0,\, \frac{4}{5}\frac{b}{2\LambW*{\frac{b}{2\rho_0}}}\right) \leq \rho_b \leq \frac{b}{2\LambW*{\frac{b}{2\rho_0}}}\,.
    \end{equation}
\end{corollary}
\begin{proof}
    
    Since $b\mapsto\dd{b}\rho_b > 0$ by \cref{lem:ode-rescaling-keys-queries} we have $\rho_b > \rho_0$ and therefore $2/(\rho_b^2 + 1) \leq 2/(\rho_0^2 + 1)$ for $b>0$. One can check numerically that $2/(\rho_0^2 + 1) \leq 1/5$. Hence,
    \begin{equation}
        \frac{4}{5}\frac{\rho_b}{2\rho_b + b} \leq \dd{b}\rho_b \leq \frac{\rho_b}{2\rho_b + b}\,.
    \end{equation}
    The simplified differential equation $\dd{b}\tilde{\rho}_b = \frac{\tilde{\rho}_b}{2\tilde{\rho}_b + b}$ with $\tilde{\rho}_0 = \rho_0$ now has the closed-form solution $\tilde{\rho}_b = b/(2 \LambW{b/(2\rho_0)})$. 
\end{proof}

\section{\pcref{thm:cmpd-attn-guarantees-non-aymptotic}}\label{subsection:proof:thm:cmpd-attn-guarantees}
Cauchy-Schwarz and the definitions of $\RQ = \rownorm{\Q}$ and $\RK = \rownorm{\K}$ imply that
\begin{align}
\min_{i\in [m], l\in [n]} \A_{il} = 
\exp\Big(\beta \min_{i\in [m], l\in [n]}\inner{\queries_i}{\keys_l}\Big)
\geq
\exp\Big(-\beta\max_{i\in [m], l\in [n]} \twonorm{\queries_i}\twonorm{\keys_l}\Big)
\geq
\exp(-\beta\RQ\RK).   
\end{align} 
Hence, to conclude that $\E\error[r] \leq 3\Vnorm \varepsilon$ for $\eps \defeq n^{-a}$, it suffices to show that
\begin{equation}\label{eq:rowwise-target-bound}
    \mathbb E\rownorm{\A - \hatA_\tau} \leq \varepsilon \sqrt{n}\exp(-\beta\RQ\RK)
\end{equation}
by \cref{prop:transformer-bound-via-A-approximation}.
Moreover, by \cref{lem:basis-change}, the rowwise bound \cref{eq:rowwise-target-bound} holds whenever the \rpnys rank parameter $r$ satisfies
    \begin{align}
        r&\geq \binom{s + d}{d}\paren{\log(\sspecnorm{\Hc}\exp(\beta\tau^2\RQ^2 + 2\beta\RQ\RK))/(\varepsilon \sqrt{n})^2}+ 1\, \qtext{for some}\label{eq:target-rank-bound}\\
        s &\geq \floor{e\beta\RQ\RK l_b(\rho)}
        \qtext{with}
        l_b(\rho) \defeq \frac{\frac{1}{e}(\rho^2 + b\rho + 1)}{\rho\LambW*{\frac{1}{e}(\rho^2 + b\rho + 1)}}\,, \quad 
        b \defeq \frac{\log\paren{\frac{1}{\varepsilon^2}}}{\beta\RQ\RK} + 2\,, \quad \text{and}\quad \rho \defeq \tau^2 \frac{\RQ}{\RK}\,.\label{eq:target-s-bound}
    \end{align}
    Thus, we will prove that our assumed constraint on $r$ \cref{eq:wildcat-r-bound} implies the rank bound \cref{eq:target-rank-bound}.

    We begin by identifying a relevant Taylor approximation order $s$ that satisfies the constraint \cref{eq:target-s-bound}.
    For each $b'\geq 0$, let $\rho_{b'}$ be the optimiser of $l_{b'}$ in \cref{lem:ode-rescaling-keys-queries}, and recall the definitions \cref{eq:kernel-temperature}
    \begin{align}
    \tau \defeq \sqrt{\frac{\RK}{\RQ}\frac{b_0}{2\LambW*{{b_0/}{(2\rho_0)}}}} \quad\text{with}\quad b_0\defeq \frac{\log(n)}{\beta\RQ\RK} + 2\,
    \end{align}
    which imply $\rho = \frac{b_0}{2\LambW*{{b_0/}{(2\rho_0)}}}$. 
    By \cref{cor:basis-change-bounds} we have $\rho_0 \leq \rho_{b_0} \leq \rho$. 
    By \cref{cor:basis-change-bounds,lambert-exponential}, we also have
    \begin{align}
    \rho_{b'} 
    \geq 
    \frac{4}{5} \frac{b'}{2\LambW*{{b'/}{(2\rho_0)}}} = \frac{4}{5}\rho_0\exp(\LambW*{{b'/}{(2\rho_0)}})
    \qtext{for all} b'\geq 0.
    \end{align}  Since $b'\mapsto\rho_{b'}$ is continuous for $b'\geq 0$ by \cref{lem:ode-rescaling-keys-queries} and
    $b'\mapsto \rho_0\exp(\LambW*{{b'/}{(2\rho_0)}})$ is coercive as $b'\to\infty$ by \cref{lambert-lower-bound}, 
    there exists a $b_\tau \geq b_0$ such that $\rho_{b_\tau} = \rho$. Fix any such $b_\tau$. 
    
    We now use the concavity of $x\mapsto x/\LambW{x}$ and the derivative  $\dd x (x/\LambW{x}) = (\LambW{x} +1)^{-1}$ (see \cref{cor:concavity_zdivLambWz}) to upper bound $l_b(\rho)$ by a linear approximation in $b$ with expansion point $b_\tau$:
    \begin{align}
        l_b(\rho) &\leq l_{b_\tau}(\rho) + \frac{\frac{1}{e}(b-b_\tau)\rho}{\rho\paren{\LambW*{\frac{1}{e}(\rho^2 + b_\tau\rho + 1)} + 1}} 
        \leq \frac{1}{e}\rho + \frac{\frac{1}{e}(b-b_\tau)}{\LambW*{\frac{1}{e}(\rho^2 + b_\tau\rho + 1)} + 1} \\\notag
        &= \frac{1}{e}\rho + \frac{1}{e}(b-b_\tau)\frac{\rho^2 -1}{2\rho^2 + b_\tau\rho}\,.
    \end{align}
    Above, the second inequality follows from \cref{lem:ode-rescaling-keys-queries} since $\rho_{b_\tau} = \rho$, and
    the equality follows from  the first order condition \cref{eq:LambW-at-optimality} characterising $\rho_{b_\tau}$.
    
    To further upper bound $l_b(\rho)$, we will consider two cases. 
    We first recall that $\rho \geq \rho_0 > 1$ and $b_0 \leq b_{\tau}$ and note that $b_0\le b$ since $\varepsilon \leq n^{-\frac{1}{2}}$. 
    Hence, if $b < b_\tau$, then
    \begin{align}
        \frac{1}{e}\rho + \frac{1}{e}(b-b_\tau)\frac{\rho^2 -1}{2\rho^2 + b_\tau\rho} \leq \frac{1}{e}\rho = \frac{1}{e} \frac{b_0}{2\LambW*{\frac{b_0}{2\rho_0}}}\leq \frac{1}{e} \frac{b}{2\LambW*{\frac{b_0}{2\rho_0}}}\,.
    \end{align}
    Alternatively, if $b \geq b_\tau$, we have 
    \begin{align}
        \frac{1}{e}\rho + \frac{1}{e}(b-b_\tau)\frac{\rho^2 -1}{2\rho^2 + b_\tau\rho} &= \frac{1}{e}\paren{\frac{b_0}{2\LambW*{\frac{b_0}{2\rho_0}}} + (b-b_\tau)\frac{\rho^2 -1}{2\rho^2 + b_\tau\rho}} \\
        &\leq \frac{1}{e}\paren{\frac{b_0}{2\LambW*{\frac{b_0}{2\rho_0}}} + (b-b_0)\frac{\rho}{b_0}} 
        = \frac{1}{e}\paren{\frac{b_0}{2\LambW*{\frac{b_0}{2\rho_0}}} + \frac{b-b_0}{2\LambW*{\frac{b}{2\rho_0}}}} 
        = \frac{1}{e} \frac{b}{2\LambW*{\frac{b_0}{2\rho_0}}}.
    \end{align}
    
    Hence, the following choice of $s$ satisfies the Taylor approximation order constraint \cref{eq:target-rank-bound}: 
    \begin{equation}\label{eq:wildcat-s}
    s 
        \defeq 
    \floor*{\beta \RQ\RK\frac{b}{2\LambW*{\frac{b_0}{2\rho_0}}}}
        = 
    \floor*{\frac{\log\paren{\frac{1}{\varepsilon^2}} + 2\beta\RQ\RK}{2\LambW*{\frac{\log(n)}{2\rho_0\beta\RQ\RK} + \frac{1}{\rho_0}}}}
        \geq
        \floor*{e\beta\RQ\RK l_b(\rho)}.
    \end{equation}
    Note, moreover, that our Taylor growth parameter $\sigma \geq s/\log(n)$.

    We will now prove that our assumed constraint on $r$ \cref{eq:wildcat-r-bound} implies the rank bound \cref{eq:target-rank-bound} with our particular choice of $s$ \cref{eq:wildcat-s}.
    By \cref{lem:binom}, we have 
    \begin{equation}
        \binom{s+d}{d} \leq \frac{1}{\sqrt{\pi}}\exp\paren{(s+d)\mathrm{Ent}\paren{\frac{s}{s+d}}} 
        \leq \frac{1}{\sqrt{\pi}}n^{(\sigma + \delta)\mathrm{Ent}\paren{\frac{\sigma}{\sigma+\delta}}}\,
    \end{equation}
    where we recall that $\delta=d/\log(n)$.
    Moreover, since 
    $\sspecnorm{\Hc} \leq n\maxnorm{\Hc} \leq n\exp(\beta \tau^{-2}\RK^2)$ by Cauchy-Schwarz, we have 
    \begin{equation}
    \log(\sspecnorm{\Hc}\exp(\beta\tau^2\RQ^2 + 2\beta\RQ\RK)/(\varepsilon\sqrt{n})^2)
        \leq
    \log\paren{\frac{1}{\varepsilon^2}}
        +
    \paren{\rho + \frac{1}{\rho} + 2}\beta\RQ\RK\,.
    \end{equation}
    In addition, $\rho^{-1} +2  \leq \rho_0^{-1} + 2 \leq 3$, and, since $\varepsilon \leq n^{-\frac{1}{2}}$, we have $\beta\RQ\RK\rho \leq \sigma \log(n)$. 
    Therefore, the runtime inflation factor is bounded as
    \begin{equation}
    \log(\sspecnorm{\Hc}\exp(\beta\tau^2\RQ^2 + 2\beta\RQ\RK)/(\varepsilon\sqrt{n})^2) 
        \leq 
        \log\paren{\frac{1}{\varepsilon^2}} + 3\beta\RQ\RK + \sigma \log(n)
        =
        \log(n^{2a + 3\gamma + \sigma}),
    \end{equation}
    confirming the sufficiency of our constraint \cref{eq:wildcat-r-bound} on $r$.
\section{\pcref{cor:guarantee-cmpd-attn-asymptotic}}\label{subsection:derived-asymptotic results}
In this section we provide general conditions on regimes in which we can provide fast attention approximation guarantees.
We note conditions under which the binomial coefficient, which is the dominating contribution to $r$, is near constant:
\begin{lemma}[Binomial coefficient growth]\label{prop:asymptotic-behaviour-binomial-coefficient}
    For $s,d\in\naturals$, we have $\binom{s+d}{d} \in n^{o(1)}$ whenever either of the following two conditions holds: 
    \begin{enumerate}
        \item $\frac{s}{\log(n)}\in o(1) \quad \text{and} \quad \frac{d}{\log(n)} \in n^{o(1/s)}$.
        \item $\frac{d}{\log(n)}\in o(1) \quad \text{and} \quad \frac{s}{\log(n)} \in n^{o(1/d)}$.
    \end{enumerate}
    In particular, $\exp\paren{(s + d)\mathrm{Ent}\paren{\frac{s}{s+d}}} \in n^{o(1)}$ under either condition.
\end{lemma}
\begin{proof}
    By \cref{lem:binom} we have
    \begin{equation}
\begin{aligned}
    \binom{s+d}{d} &
    \leq \frac{1}{\sqrt{ \pi}}\left(1 + \frac{s}{d}\right)^d \left(1 + \frac{d}{s}\right)^s
    =
    \frac{1}{\sqrt{ \pi}}\exp\paren{(s+d)\mathrm{Ent}\paren{\frac{s}{s+d}}}\,.
\end{aligned}
\end{equation}
     To conclude $\binom{s+d}{d} \in n^{o(1)}$ it therefore suffices to establish
    \begin{equation}
        -s \log\paren{\frac{s}{s+d}}  -d \log\paren{\frac{d}{s+d}} \in o(\log(n))\,.
    \end{equation}
    First, consider $s \in o(\log(n))$. Then,
    \begin{equation}
        \exp\paren{-d \log\paren{\frac{d}{s+d}}} = \paren{1 + \frac{s}{d}}^d \leq \exp(s) \in n^{o(1)}\,.
    \end{equation}
    Hence, it suffices to have $-s \log\paren{\frac{s}{s+d}} \in o(\log(n))$ to prove the claim. We have
    \begin{equation}
    \begin{aligned}
        -s \log\paren{\frac{s}{s+d}} \in o(\log(n))\quad &\Leftrightarrow\quad \frac{s+d}{s} \in \exp\paren{\frac{o(\log(n))}{s}} \\
        &\Leftrightarrow\quad \frac{d}{\log(n)} \in \frac{s}{\log(n)}\paren{\exp\paren{o(1)\frac{\log(n)}{s}}-1}\\
        &\Leftarrow \quad \frac{d}{\log(n)} \in \exp\paren{o(1)\frac{\log(n)}{s}}\,.
    \end{aligned}
    \end{equation}
    In the final equation it was used that $s/\log(n) \in o(1)$, and for any $f(n)\in \omega(1)$ it holds that 
    \begin{equation}
        \exp(o(f(n)))/f(n) = \exp\paren{f(n)\paren{o(1) + \frac{\log(1/f(n))}{f(n)}}} = \exp(o(f(n)))\,.
    \end{equation}
    The second case follows by a symmetric argument.
\end{proof}
\paragraph{\pcref{cor:guarantee-cmpd-attn-asymptotic}}
\begin{proof}
    For any $a(n)\in o\paren{\frac{\log(1/\gamma(n))}{\max\{\log(\delta(n)), 1\}}}\cap n^{o(1)}$ it holds in particular that $a(n) \in  o\paren{\log(1/\gamma(n))}$, since $\max\{\log(\delta(n)), 1\}\in \Omega(1)$ is asymptotically bounded from below. Since we further assumed that $\gamma(n) \in o(1)$, it holds by \cref{prop:asymptotic-behaviour-taylor-order-parameter} that
    \begin{equation}
        \sigma(n) \leq \frac{a(n) + \gamma(n)}{c_1\log\paren{1 + \frac{1}{2\rho_0\gamma(n)} + \frac{1}{\rho_0}}} \in \bigO\paren{\frac{a(n)}{\log(1/\gamma(n))}} \subseteq o(1)\,.
    \end{equation}
    Let now $a(n) \in o\paren{\frac{\log(1/\gamma(n))}{\max\{\log(\delta(n)), 1\}}}\cap n^{o(1)}$ be fixed. Then, there exists a sequence $\alpha(n) \in o(1)$ such that $a(n) \leq \alpha(n) \paren{\frac{\log(1/\gamma(n))}{\max\{\log(\delta(n)), 1\}}}$ and we have
    \begin{align}
        \delta(n) &\leq \exp(\max\{\log(\delta(n), 1\}) 
        \leq \exp\paren{\alpha(n)\frac{\log(1/\gamma(n))}{a(n)}} 
         \leq \exp\paren{\alpha(n)\frac{1}{c\sigma(n)}} \in \exp\paren{o(1)/\sigma(n)}\,.
    \end{align}
    Therefore, by \cref{prop:asymptotic-behaviour-binomial-coefficient},  $n^{(\sigma(n) + \delta(n))\mathrm{Ent}\paren{\frac{\sigma(n)}{\sigma(n) + \delta(n)}}} \in n^{o(1)} $. Finally, we know from $a(n) + \gamma(n) \in n^{o(1)}$ that the logarithmic runtime inflation term is near constant $\log(n^{2a + \sigma + \gamma}) \in n^{o(1)}$. Therefore, any
    \begin{equation}
        r \geq 1 + \frac{1}{\sqrt{\pi}}n^{(\sigma + \delta)\mathrm{Ent}\paren{\frac{\sigma}{\sigma + \delta}}}\log\paren{n^{2a + \sigma + 3\gamma}} \in n^{o(1)}
    \end{equation}
    suffices to achieve $\mathbb E\error \leq 3\Vnorm n^{-a(n)}$.

    Alternatively, assume that $\delta(n)\in o(1)$ and $\gamma(n) \in \Omega(1)\cap n^{o(1/d)}, a(n) \in n^{o(1/d)}$. In particular, $a(n) + \gamma(n) \in n^{o(1/d)}$. Since $\gamma(n) \in \Omega(1)$, it holds by \cref{prop:asymptotic-behaviour-taylor-order-parameter} that
    \begin{equation}
        \sigma(n) \in \bigO(a(n) + \gamma(n))\subseteq n^{o(1/d)}\,.
    \end{equation}
    Therefore, we conclude as before that there exists a runtime function $r \in n^{o(1)
    }$ that satisfies
    \begin{equation}
        r \geq 1 + \frac{1}{\sqrt{\pi}}n^{(\sigma + \delta)\mathrm{Ent}\paren{\frac{\sigma}{\sigma + \delta}}}\log\paren{n}n^{o(1/d)} \in n^{o(1)}\,.
    \end{equation}
\end{proof}
\begin{lemma}[Asymptotic behaviour of Taylor approximation order]\label{prop:asymptotic-behaviour-taylor-order-parameter}
    The order parameter of the Taylor approximation $\T^s$ has the following asymptotic behaviour:
    \begin{equation}
        \sigma(n) \defeq \frac{s(\varepsilon = n^{-a(n)})}{\log(n)} = \frac{a(n) + \gamma(n)}{\LambW*{\frac{1}{2\rho_0\gamma(n)} + \frac{1}{\rho_0}}} \in \begin{cases}
            \bigO(a(n)\log(1/\gamma(n))^{-1}) \quad \text{if}\quad \gamma(n) \in o(1) \\
            \bigO(a(n) + \gamma(n))\quad \text{if}\quad \gamma(n) \in \Omega(1)\,.
        \end{cases}
    \end{equation}
\end{lemma}
\begin{proof}
    First, assume that $\gamma(n) \in o(1)$. Then, by definition of $\sigma(n)$ we have
    \begin{equation}
        \frac{s(\varepsilon = n^{-a})}{\log(n)} = \frac{a + \gamma(n)}{\LambW*{\frac{1}{2\rho_0\gamma(n)} + \frac{1}{\rho_0}}} \leq \frac{a + \gamma(n)}{c_1\log\paren{1 + \frac{1}{2\rho_0\gamma(n)} + \frac{1}{\rho_0}}} \in \bigO(a(n)\log(1/\gamma(n))^{-1})\,.
    \end{equation}
    because $\log\paren{1 + \frac{1}{2\rho_0\gamma(n)} + \frac{1}{\rho_0}} \in \omega(1)$. Here, we used that $\LambW{x} \geq c_1 \log(1 + x)$ with $c_1 = 0.6321$ \citep{orabona2019modern}. Conversely, assume that $\gamma(n) \in \Omega(1)$. In this case, 
    $\LambW*{\frac{1}{2\rho_0\gamma(n)} + \frac{1}{\rho_0}} \in \bigO(1)$ and therefore
    \begin{equation}
        \frac{s(\varepsilon = n^{-a})}{\log(n)} \in \bigO(a(n) + \gamma(n))\,.
    \end{equation}
\end{proof}

\section{Proof of \cref{tab:comparison-approximate-attention-guarantees}: Practical approximation guarantees}\label{proof-tab:comparison-approximate-attention-guarantees}
In this section we derive the guarantees of   \cref{tab:comparison-approximate-attention-guarantees} given $m=n$, 
bounded dimension $d \in \bigO(1)$, 
bounded entries \begin{align}\label{eq:bounded-entries}
\beta\RQ^2, \beta\RK^2 \leq R^2\in O(1),
\end{align} and $\bigO(dn^{1+t})$ runtime.
\subsection{\catt guarantee}
The stated result follows immediately from the more precise guarantee in \cref{cor:table-comparison-rates}.
\begin{corollary}[\cattnolink error as a function of runtime]\label{cor:table-comparison-rates}
    Suppose $m=n$, $t \in (0, 1)$, $d \in \bigO(1)$, and $\beta\RQ\RK\leq R\in O(1)$. Define $\kappa \defeq 2\rho_0 + 1$. Let $r = n^{t/2}$ so that the runtime of \catt lies in $O(n^{1 + t}+dn^{1+t/2})$. Then there exists a constant $C > 0$ such that 
    \begin{equation}
        \mathbb E \error \leq C \frac{\log(n)}{n^{0.14 t\paren{1 + \log\paren{1 + \log(n)/(\kappa R)}}}}\cdot\maxnorm{\V}\,.
    \end{equation}
\end{corollary}
\begin{proof} 
    Matching the exponents in \cref{thm:cmpd-attn-guarantees-non-aymptotic}, our goal is finding a function $\sigma(n) = s/\log(n)$ such that
    \begin{equation}\label{eq:exponent-comparison}
        (\sigma + \delta)\mathrm{Ent}\paren{\frac{\sigma}{\sigma + \delta}} \leq t/2 - \lambda(n)\,
    \end{equation}
    with
    \begin{equation}
        \lambda(n) = \frac{\log((2a + \sigma + 3\gamma)\log(n))}{\log(n)}\,.
    \end{equation}
    Since $\gamma\in o(1)$ we have from \cref{prop:asymptotic-behaviour-taylor-order-parameter} that $\sigma \in \bigO(a/\log(1/\gamma)) \subseteq o(a)$. Furthermore, since $t < 1$ we have for our claimed error rate $a < 0.29 \log(e +\log(n)/(\kappa R)) \in o(\log(n))$. Therefore, for $n$ large enough
    \begin{equation}
    \lambda(n) \leq \frac{\log(2a(1 + o(1))\log(n))}{\log(n)} \leq \frac{\log(\log(n)^2)}{\log(n)} \in o(1)\,.
    \end{equation}
    We therefore know that $\lambda$ remains small relative to the other exponents in \cref{eq:exponent-comparison}. We proceed with simplifying the left-hand side as in \cref{prop:asymptotic-behaviour-binomial-coefficient} as
    \begin{equation}
        (\sigma + \delta)\mathrm{Ent}\paren{\frac{\sigma}{\sigma + \delta}} \leq \sigma \log\paren{{\frac{\sigma + \delta}{\sigma}}} + \sigma \,.
    \end{equation}
    Setting the right-hand side equal to $t/2 - \lambda(n)$ we find
    \begin{equation}
        \frac{\sigma}{\delta} \log\paren{{e\frac{\sigma + \delta}{\sigma}}} = \frac{t - \lambda}{2\delta}\,.
    \end{equation}
    With \cref{app:lem:inverse-function-pair-LambertW_-1} we can invert this relationship and find the following expression for the order parameter $\sigma$ as a function of $t$:
    \begin{equation}
        \sigma = \delta g\paren{\frac{t - \lambda}{2\delta}}
    \end{equation}
    where $g(y) = \frac{y}{y + \LambWm*{-\exp(-y-1)y}}$. Next, we relate $\sigma$ to the decay rate $n^{-a}$. Re-arranging \cref{eq:taylor-order-explicit} we have
    \begin{equation}
        a = \LambW*{\frac{1}{2\rho_0\gamma} + \frac{1}{\rho_0}}\sigma - \gamma\,.
    \end{equation}
    Next, we use $\LambW*{x} \geq 0.6321\log(1 + x)$ \citep[Theorem C.3]{orabona2019modern} to state
    \begin{align}
        \LambW*{\frac{1}{2\rho_0\gamma} + \frac{1}{\rho_0}} &\geq 0.6321\log\paren{1 + \frac{1}{2\rho_0\gamma} + \frac{1}{\rho_0}} \\
        &= 0.6321\paren{\log\paren{1 + \frac{1}{(2\rho_0 + 1)\gamma}} + \log\paren{1 + \frac{1}{\rho_0}}}\\
        &\geq 0.6321\paren{\log\paren{1 + \frac{1}{(2\rho_0 + 1)\gamma}} + 1}.
    \end{align}
    We therefore obtain the following lower bound on the decay rate in $n^{1 + t}$ time:
    \begin{equation}
        a \geq 0.6321\paren{\log\paren{1 + \frac{1}{(2\rho_0 + 1)\gamma}} + 1}\delta g\paren{\frac{t - \lambda}{2\delta}} - \gamma\,.
    \end{equation}
    Since $(t - \lambda)\log(n)/2d \in \omega(1)$, we have for $n$ large enough that $(t - \lambda)\log(n)/2d \geq 1$. 
    We now use the convexity of $g$ to obtain the following
    \begin{align}
      \delta g\paren{\frac{t - \lambda}{2\delta}} &= \frac{t-\lambda}{2}\frac{2\delta}{t-\lambda} g\paren{\frac{t - \lambda}{2\delta}} \\
      &= \frac{t-\lambda}{2}\paren{\frac{2\delta}{t-\lambda} - 0} \paren{g\paren{\frac{t - \lambda}{2\delta}} - g(0)} \\
      &\geq \frac{t-\lambda}{2}(1 - 0) (g(1) - g(0))  = \frac{t-\lambda}{2} g(1)\,.
    \end{align}
    With $0.6321 *g(1)/2 \geq 0.14$ we obtain the final closed form lower-bound on the decay rate
    \begin{align}
        a &\geq 0.6321\paren{\log\paren{1 + \frac{1}{(2\rho_0 + 1)\gamma}} + 1}\frac{t-\lambda}{2} g(1) - \gamma \\
        & \geq 0.14 \paren{1 + \log\paren{1 + \frac{1}{(2\rho_0 + 1)\gamma}}}t - \lambda - \gamma\,.
    \end{align}
    Finally, the decay rate simplifies to
    \begin{equation}
        n^{-a} \leq \frac{n^{\gamma + \lambda}
        }{n^{0.14 \paren{1 + \log\paren{1 + \frac{1}{(2\rho_0 + 1)\gamma}}}t}} \leq \frac{\exp(2R)\log(n)}{n^{0.14 \paren{1 + \log\paren{1 + \frac{1}{(2\rho_0 + 1)\gamma}}}t}}\,.
    \end{equation}
\end{proof}
\subsection{Thinformer guarantee}
The runtime analysis of \citet[Sec.~4.1]{carrell2025low} allows for a maximum coreset of size $\nout = \Theta(n^{t})$ in time $O(dn^{1+t})$.
Plugging this choice into the error analysis of \citet[Thm.~2]{carrell2025low} yields the following constant probability bound on $\error$ (up to constants): 
\begin{align}
\frac{\sqrt{d\log(R^2\maxnorm{V})}\exp(2R^2)\rownorm{\V}\log(n)}{n^t}.
\end{align}

\subsection{BalanceKV guarantee}
The analysis of \citet[Thm.~3.1]{han2025streaming} guarantees that BalanceKV with batch size $B$ provides a high probability bound of order 
\begin{align}\label{eq:balancekv-error}
\frac{\sqrt{d}\log(d n)\log_2(n/B) \exp(2R^2) \fronorm{V}}{B}
\end{align}
on $\error$ in order $d n B \log_2(n/B)$ time.
Plugging $B=n^t/\log_2(n)$ into \cref{eq:balancekv-error} to achieve $dn^{1+t}$ runtime yields the result.

\subsection{KDEformer guarantee}
The bounded entries assumption \cref{eq:bounded-entries} implies that, for all $i,j\in[n]$, 
\begin{align}
\A_{ij} 
    &\in 
[e^{-R^2},e^{R^2}], 
    \qquad
1
    \leq
|(\D^{-1}\A)_{ij}|  
    \leq
\frac{e^{R^2}}{e^{R^2}+(n-1)e^{-R^2}}
    =
\frac{e^{2R^2}}{e^{2R^2}+n-1}, \\ \label{eq:DinvAnorms}
\rownorm{\D^{-1}\A}^2
    &\leq
\frac{e^{4R^2}n}{(e^{2R^2}+n-1)^2} 
    \in 
O\paren{\frac{1}{n}},
\qtext{and}
1
    \leq
\opnorm{\D^{-1}\A}^2
    \leq
\fronorm{\D^{-1}\A}^2
    \leq
n\rownorm{\D^{-1}\A}^2
    \in
O(1).
\end{align}
In addition, the bounded dimension assumption implies that
\begin{align}
\opnorm{\V}^2
    \leq
\fronorm{\V}^2
    \leq
d\rownorm{\V^\top}^2
    \leq 
d\opnorm{\V}^2
    \in
O(\opnorm{\V}^2).
\end{align} 
Hence, the runtime analysis of \citet[Thm.~3.5]{zandieh2023kdeformer} guarantees that KDEformer provides
a high probability bound of order
\begin{align}\label{eq:kdeformer-error}
\eps \opnorm{\D^{-1}\A}\opnorm{\V}
    \in
\Theta(\eps\opnorm{\V})
\end{align}
on $\error$ 
in order 
\begin{align}
\frac{d\big(n^{1+\xi}+n\log(n)(\frac{\fronorm{\D^{-1}\A}^2}{\opnorm{\D^{-1}\A}^2} + \frac{\fronorm{\V}^2}{\opnorm{\V}^2})\big)}{\eps^2}
    \in
\Theta\bigg(\frac{dn^{1+\xi}}{\eps^2} \bigg)
\end{align}
time.
Plugging $\eps=n^{\xi/2}/n^{t/2}$ into \cref{eq:kdeformer-error} to achieve order $dn^{1+t}$ runtime yields the result.
\subsection{HyperAttention guarantee}
The bounded entries assumption \cref{eq:bounded-entries} implies that
\begin{align}\label{eq:kappa}
1
    \leq
\kappa
    \defeq
\frac{\max_{i\in[n]}\sum_{j\in[n]}\A_{ij}}{\min_{i\in[n]}\sum_{j\in[n]}\A_{ij}}
    \leq
\frac{ne^{R^2}}{ne^{-R^2}}
    =
e^{2R^2}.
\end{align}
Using \cref{eq:DinvAnorms,eq:kappa}, we find that the Hyperattention without masking analysis of \citet[Thm.~1, Lem.~1, and Lem.~2]{han2024hyperattention} requires order
\begin{align}
dn\bigg(\frac{\kappa^7 n\rownorm{\D^{-1}\A}^2\log(n)}{\eps^6}+\frac{d \kappa^2 \fronorm{\D^{-1}\A}^2/\opnorm{\D^{-1}\A}^2}{\eps^2}\bigg)
    \in
\Theta(dn\log(n)/\eps^6)
\end{align}
time to guarantee an order
\begin{align}\label{eq:hyperattention-error}
\eps \opnorm{\D^{-1}\A}\opnorm{\V}
    \in
\Theta(\eps\opnorm{\V})
\end{align}
high-probability bound on $\error$.
Plugging $\eps=(\log n)^{1/6}/n^{t/6}$ into \cref{eq:hyperattention-error} to achieve order $dn^{1+t}$ runtime yields the result.

\section{\pcref{prop:kernel-core-inverse-update}}\label{app:proof:prop:kernel-core-inverse-update}
\begin{proposition}[Recursive update of kernel inverse]\label{prop:kernel-core-inverse-update}
    Let $h(\coreset, \coreset)$ be invertible, let $\keys \in \mathbb R^d$ with 
    \begin{equation}
        \hres(\keys, \keys) \defeq h(\keys, \keys) - h(\keys, \coreset)h(\coreset, \coreset)^{-1}h(\coreset, \keys) > 0\,.
    \end{equation} 
    Define the vector
    \begin{equation}
        \g \defeq \frac{1}{\sqrt{\hres(\keys, \keys)}}(h(\coreset, \coreset)^{-1}h(\coreset, \keys), -1)^\top
    \end{equation}
    Then, $\begin{pmatrix}
            h(\coreset, \coreset) & h(\coreset, \keys) \\
            h(\keys, \coreset) & h(\keys, \keys)
        \end{pmatrix}$ is invertible and
    \begin{equation}
        \begin{pmatrix}
            h(\coreset, \coreset) & h(\coreset, \keys) \\
            h(\keys, \coreset) & h(\keys, \keys)
        \end{pmatrix}^{-1}
        =
        \begin{pmatrix}
            h(\coreset, \coreset)^{-1} & \boldzero_r \\
            \boldzero_r^\top & 0
        \end{pmatrix}
        + \g \g^\top
    \end{equation}
\end{proposition}
\begin{proof}
    We derive the recursion using Gaussian elimination. We start with
    \begin{equation}
        \paren{\begin{array}{cc|cc}
        h(\coreset, \coreset) & h(\coreset, \keys) & \boldone_{r\times r} & \boldzero_r \\
            h(\keys, \coreset) & h(\keys, \keys) & \boldzero_r^\top & 1
        \end{array}}
    \end{equation}
    Multiplying the upper $r$ rows with $h(\coreset, \coreset)^{-1}$ from the left yields
    \begin{equation}
        \paren{\begin{array}{cc|cc}
        \boldone_{r\times r} & h(\coreset, \coreset)^{-1}h(\coreset, \keys) & h(\coreset, \coreset)^{-1} & \boldzero_r \\
            h(\keys, \coreset) & h(\keys, \keys) & \boldzero_r^\top & 1
        \end{array}}
    \end{equation}
    Subtracting $-h(\keys, \coreset)$ times the first $r$ rows from the last row produces the following matrix:
    \begin{equation}
        \paren{\begin{array}{cc|cc}
        \boldone_{r\times r} & h(\coreset, \coreset)^{-1}h(\coreset, \keys) & h(\coreset, \coreset)^{-1} & \boldzero_r \\
            \boldzero_r^\top & \hres(\keys, \keys) & -h(\keys, \coreset)h(\coreset, \coreset)^{-1} & 1
        \end{array}}
    \end{equation}
    Next, divide the last row by $\hres(\keys, \keys)$ and subtract $h(\coreset, \coreset)^{-1}h(\coreset, \keys)$ of the last row from the first block to obtain
    \begin{equation}
        \paren{\begin{array}{cc|cc}
        \boldone_{r\times r} & \boldzero_r & h(\coreset, \coreset)^{-1} + \frac{h(\coreset, \coreset)^{-1}h(\coreset, \keys)h(\keys, \coreset)h(\coreset, \coreset)^{-1}}{\hres(\keys, \keys)} & \frac{-h(\coreset, \coreset)^{-1}h(\coreset, \keys)}{\hres(\keys, \keys)} \\
            \boldzero_r^\top & 1 & \frac{-h(\keys, \coreset)h(\coreset, \coreset)^{-1}}{\hres(\keys, \keys)} & \frac{1}{\hres(\keys, \keys)}\,.
        \end{array}}
    \end{equation}
    The right hand side is the inverse of $h(\coreset\cup \{\keys\}, \coreset\cup \{\keys\})$. Furthermore, we find that indeed
    \begin{equation}
        \g\g^\top = \frac{1}{\hres(\keys, \keys)}
        \begin{pmatrix}
            h(\coreset, \coreset)^{-1}h(\coreset, \keys)h(\keys, \coreset)h(\coreset, \coreset)^{-1} & -h(\coreset, \coreset)^{-1}h(\coreset, \keys) \\
            -h(\keys, \coreset)h(\coreset, \coreset)^{-1} & 1\,.
        \end{pmatrix}
    \end{equation}
\end{proof}
\section{Properties of the Lambert W Function}\label{sec:lambert-w-identities}
A useful function for our statements is the Lambert W function, also known as the product logarithm:
\begin{definition}[Lambert W function]
    The principal branch of the Lambert W function $w = \LambW{z}$ is the unique solution $w \in (-1, \infty)$ to the equation $w\exp(w) = z$ for $z > -\frac{1}{e}$.
\end{definition}
From the definition one immediately gets the following identity:
\begin{lemma}[Lambert W exponential]\label{lambert-exponential}
The Lambert W function satisfies $\exp(\LambW{z}) = \frac{z}{\LambW{z}}$ for all $z \neq 0$, and $\exp(\LambW{0}) = 1$.
\end{lemma}
By implicit differentiation one further finds the following:
\begin{lemma}[Lambert W derivative]
    The Lambert W function is increasing on $(-\frac{1}{e}, \infty)$ and has the derivative\begin{equation}\label{eq:derivative-lambW}
        \dd z \LambW{z} = \begin{cases}
            \frac{\LambW{z}}{z(1+ \LambW{z})} \quad &z \neq 0\\
            1 &z = 0
        \end{cases}
    \end{equation}
    In addition, the ordinary differential equation $\dd t x(t) = \frac{x}{x + t}$ with $x(0) = x_0$ has the solution
    \begin{equation}
        x(t) = \frac{t}{\LambW{\frac{t}{x_0}}}\,.
    \end{equation}
\end{lemma}
\begin{proof}
    $z\mapsto \LambW{z}$ is increasing as the inverse of the increasing function $w \mapsto w\exp(w)$. By the definition of the Lambert W function we have for all $z \in (-\frac{1}{e}, \infty)$
    \begin{equation}
        \dd z \paren{\LambW{z} \exp(\LambW{z})} = (1+ \LambW{z})\dd z(\LambW{z})\exp(\LambW{z}) = 1\,.
    \end{equation}
    Rearranging for $\dd z(\LambW{z})$ gives
    \begin{equation}
        \dd z(\LambW{z}) = \frac{1}{\exp(\LambW{z}) (1+ \LambW{z})}\,.
    \end{equation}
    We find $\dd z(\LambW{z})\vert_{z=0} = 1$ by direct evaluation. Invoking $\exp(\LambW{z}) = \frac{z}{\LambW{z}}$ for $z\neq 0$ gives \cref{eq:derivative-lambW}. Next, applying the derivative formula to $x(t)$ yields
    \begin{align}
        \dd t x(t) &= \dd t x_0 \exp\paren{\LambW*{\frac{t}{x_0}}} 
        = x_0 \exp\paren{\LambW*{\frac{t}{x_0}}} \frac{1}{\exp\paren{\LambW*{\frac{t}{x_0}}}\paren{1+ \LambW*{\frac{t}{x_0}}}}\frac{1}{x_0}\\
        &= \frac{1}{1+ \LambW*{\frac{t}{x_0}}} 
        = \frac{x(t)}{x(t) + t}\,.
    \end{align}
    Furthermore, $x(0) = x_0 \exp(\LambW{0}) = x_0$.
\end{proof}
\begin{lemma}[Derivative of the Lambert W exponential]\label{cor:concavity_zdivLambWz}
    For $z > 0$, the function $\frac{z}{\LambW{z}}$ is concave and has derivative $\dd z \paren{\frac{z}{\LambW{z}}} = ( \LambW{z} + 1)^{-1}$.
\end{lemma}
\begin{proof}
    Using \cref{eq:derivative-lambW} and $\frac{z}{\LambW{z}} = \exp(\LambW{z})$, we find
    \begin{equation}
        \dd z \paren{\frac{z}{\LambW{z}}} = \dd z \exp(\LambW{z}) = \frac{\LambW{z}}{z(\LambW{z}+1)}\exp(\LambW{z}) = \frac{1}{\LambW{z} + 1}\,.
    \end{equation}
    In addition, $z\mapsto (1 + \LambW{z})^{-1}$ is decreasing for $z > 0$, and consequently $\frac{z}{\LambW{z}}$ is concave. 
\end{proof}
We use the following logarithmic lower bound on the Lambert W function, proved by \citet[Thm.~C.3]{orabona2019modern}:
\begin{lemma}[Lambert W lower bound {\citep[Thm.~C.3]{orabona2019modern}}]\label{lambert-lower-bound}
    It holds for $z\geq 0$ that $\LambW{z} \geq 0.6321 \log(1 + z)$. 
\end{lemma}

For numerical simulations we want a stable estimate of the Lambert W function and we find that standard implementations in Python packages like SciPy may be insufficient and not parallelised. A better estimate is obtained via the iterations proposed by \citet{loczi2022guaranteed}.
\begin{theorem}[Fast Lambert W calculation {\citep[Thms.~2.4 and 2.9]{loczi2022guaranteed}}]
    For $z > 0$, let
    \begin{equation}
        \beta_0 = \begin{cases}
            \log(z) - \log\log(z) \quad \text{for}\quad z>e\\
            \exp(\log(z) - 1) \quad \text{for}\quad z<e
        \end{cases}
    \end{equation}
    and define the iteration
    \begin{equation}
        \beta_{n+1} = \frac{\beta_n}{1 + \beta_n}(1 + \log(z) - \log(\beta_n))\,.
    \end{equation}
    Then
    \begin{equation}
        0 < \beta_n - \LambW{z} < \max\left(0.32^{(2^n)}, \frac{1}{3}0.633^{(2^n)}\right)\,.
    \end{equation}
\end{theorem}

We will also use the following identity and bounds concerning the secondary branch $W_{-1}$ of the Lambda W function.
\begin{lemma}[Lambda W secondary branch properties]\label{app:lem:inverse-function-pair-LambertW_-1}
    Define
    \begin{equation}
        b: \mathbb R_{>0} \to \mathbb R_{>0}, \quad z\mapsto z\log\paren{e\frac{z+1}{z}}\,
    \end{equation}
    and
    \begin{equation}
        g: \mathbb R_{>0} \to \mathbb R_{>0}, \quad y\mapsto \defeq -\frac{y}{y + \LambWm*{-\exp(-y-1)y}}\,.
    \end{equation}
    where $\LambWm{\bullet}$ is the secondary branch of the Lambert-W function with input range $(-1/e, 0)$.
    Then, $b \circ g(y) = y$ and $g(y) \leq y$. 
\end{lemma}
\begin{proof}
    Let $w \defeq \LambWm*{-\exp(-y-1)y}$. First, computing the argument of the logarithm we find:
    \begin{align}
        \frac{g(y)+1}{g(y)} = \frac{-y/(y+w) + 1}{-y/(y+w)} = \frac{-y + y+ w}{-y} = \frac{-w}{y}\,.
    \end{align}
    Using the definition of the Lambert-W function $\LambWm{x}\exp(\LambWm{x}) = x$ we find
    \begin{equation}
       -w = -\LambWm*{-\exp(-y-1)y} = \exp(-y-1)y\exp(-w)\,.
    \end{equation}
    Therefore,
    \begin{equation}
        \log(-w)  = \log(y) - y - 1 - w \quad\implies \quad \log\paren{\frac{g(y)+1}{g(y)}} = \log\paren{\frac{-w}{y}} = -y - 1 - w\,.
    \end{equation}
    Consequently,
    \begin{equation}
        b(g(y)) = \frac{-y}{y + w} \paren{\log\paren{\frac{g(y)+1}{g(y)}} + 1} = \frac{-y}{y + w}\paren{-y -1 -w + 1} = y\,.
    \end{equation}
\end{proof}
\section{Supplementary Experiment Details}\label{app:experiment-details}
All experiments were run using Python 3.12.12 on an Ubuntu 22.04.5 LTS server with a single NVIDIA A100 GPU (80 GB memory, CUDA 13.0, driver version 580.126.09), two $48$-core AMD EPYC 7V13 processors, and 220 GB RAM.

\subsection{Supplementary details for \cref{sec:biggan}}
The \cref{tab:biggan} experiment was run using PyTorch 2.10.0.dev20251019+cu129. 
We used CUDA events to time the forward pass through each (approximate) \texttt{attention-matrix} layer after initializing the GPU with $20$ warm-up batches. The implementations and settings for all methods other than \catt were taken from \url{https://github.com/microsoft/thinformer}, and our experiment builds on this open-source repository.
\subsection{Supplementary details for \cref{sec:t2t}}
The \cref{tab:t2t} experiment was run using PyTorch 2.10.0.dev20251019+cu129. 
Timings were based on the first 50 batches of the ImageNet 2012 validation set. We used CUDA events to time the forward pass through each (approximate) \texttt{attention\_layer} after initializing the GPU with $10$ warm-up batches. The implementations and settings for all methods other than \catt were taken from \url{https://github.com/microsoft/thinformer}, and our experiment builds on this open-source repository.
\subsection{Supplementary details for \cref{sec:kvcache}}
The \cref{tab:longbenche} experiment was run using PyTorch 2.8.0+cu128 and \texttt{kvpress} version 0.3.0. The implementations and settings for BalanceKV and Uniform were taken from \url{https://github.com/ksheth96/BalanceKV}. The implementations and default settings of all other methods save \catt were taken from \url{https://github.com/NVIDIA/kvpress}, and our experiment builds on this open-source repository.

Our KV cache compression experiments focus on the memory reduction benefits of compression. When memory is the primary bottleneck (as is often the case on resource-constrained devices or for especially large contexts), one is typically willing to incur additional runtime costs for improved memory efficiency. Fortunately, we find that the \compresskv overhead is small relative to leading alternatives. For example, to process $32$k tokens with $75\%$ compression, the prefill time with SnapKV vs.\ \compresskv is $3.38$s vs.\ $3.43$s (2\% overhead).
\subsection{Supplementary details for \cref{flash}}\label{supp_flash}
The \cref{fig:flash} experiment was run using PyTorch 2.12.0+cu130. 
For each $n$, we reported the median runtime and mean approximation error over $100$ replicates of the experiment. 
We repeat this experiment with varying rank and bin count parameters $r \in \{64, 128, 256, 512\}$ and $B \in \{2, 16, 64\}$ and display the runtime vs.\ accuracy curves in \cref{fig:ablation-r-B-selection}. %
\begin{figure}[h!]
    \centering
    \includegraphics[width=0.8\linewidth]{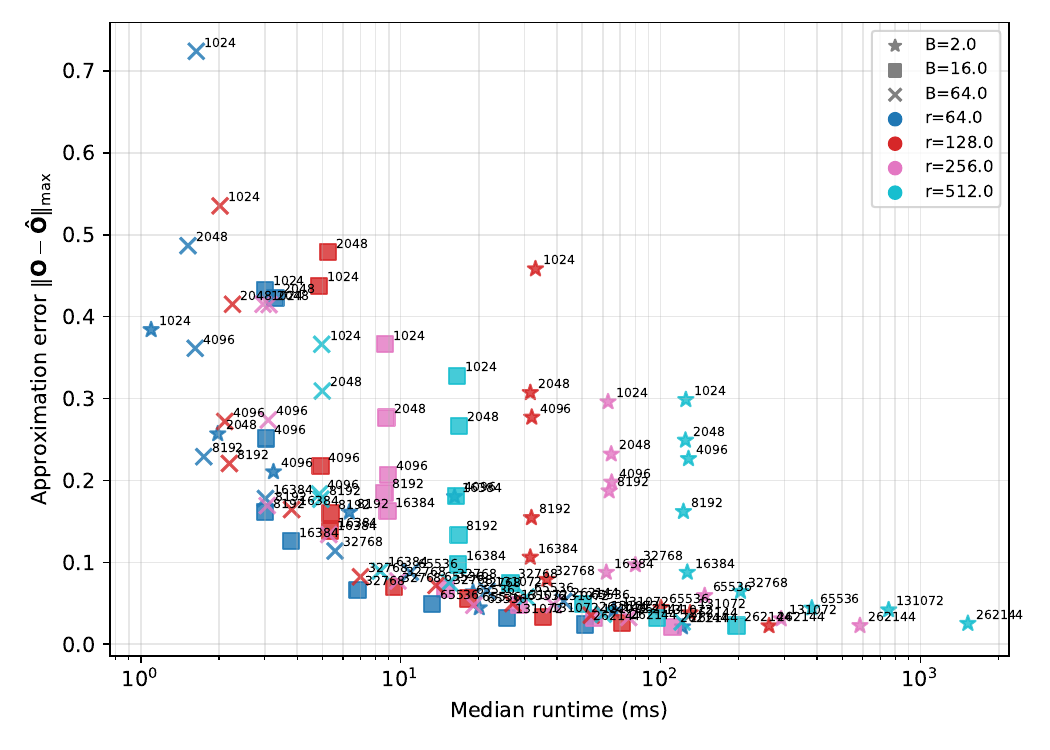}
    \vspace{-0.4cm}
    \caption{Time-accuracy tradeoff curves for \catt with varying rank and bin count parameters $(r,B)$.}
    \label{fig:ablation-r-B-selection}
\end{figure}

\end{document}

